\documentclass[wcp]{jmlr}

\usepackage[utf8]{inputenc} 
\usepackage[T1]{fontenc}    
\usepackage{url}            
\usepackage{booktabs}       
\usepackage{nicefrac}       
\usepackage{microtype}      
\usepackage{color}
\usepackage{comment}
\usepackage{lmodern}
\usepackage{graphicx}
\usepackage{hyperref}
\usepackage{amsmath,amsfonts,amssymb}
\usepackage{libs/mathletters}
\usepackage{mathrsfs}
\usepackage{longtable}

\usepackage{libs/algorithm,libs/algorithmic} 
\renewcommand{\algorithmicrequire}{\textbf{Input:}}
\renewcommand{\algorithmicensure}{\textbf{Output:}}

\graphicspath{{.}{Figs/}}

\usepackage[colorinlistoftodos, textwidth=4cm, shadow]{todonotes}

\newcommand{\beq}{\begin{equation}}
\newcommand{\eeq}{\end{equation}}
\newcommand{\beqa}{\begin{eqnarray}}
\newcommand{\eeqa}{\end{eqnarray}}
\newcommand{\beqan}{\begin{eqnarray*}}
\newcommand{\eeqan}{\end{eqnarray*}}
\newcommand{\als}[1]{ \begin{align*} #1  \end{align*}}
\newcommand{\sk}{\nonumber\\}

\newtheorem{assumption}{Assumption}

\newcommand{\argmax}{\mathop{\mathrm{argmax}}}
\newcommand{\Argmax}{\mathop{\mathrm{Argmax}}}
\newcommand{\argmin}{\mathop{\mathrm{argmin}}}

\newcommand{\bA}{\mathbb{A}}
\newcommand{\bI}{\mathbb{I}}

\newcommand{\UCRL}{{\textcolor{red!50!black}{\texttt{UCRL2}}}}
\newcommand{\UCRLLaplace}{{\textcolor{red!50!black}{\texttt{UCRL2-L}}}}
\newcommand{\CUCRL}{{\textcolor{red!50!black}{\texttt{C-UCRL}}}}
\newcommand{\ClusteringAlgo}{{\textcolor{red!50!black}{\texttt{ApproxEquivalence}}}}
\newcommand{\bH}{\beta'}
\newcommand{\bW}{\beta}
\newcommand{\ep}{\hfill$\blacksquare$}

\jmlrvolume{101}
\jmlryear{2019}
\jmlrworkshop{ACML 2019}

\title[Model-Based RL Exploiting Equivalences]{Model-Based Reinforcement Learning Exploiting State-Action Equivalence}

 \author[Asadi et al.]{\Name{Mahsa Asadi} \Email{mahsa.asadi@inria.fr}\\
   \Name{Mohammad Sadegh Talebi} \Email{sadegh.talebi@inria.fr}\\
   \Name{Hippolyte Bourel} \Email{hippolyte.bourel@ens-rennes.fr}\\
   \Name{Odalric-Ambrym Maillard} \Email{odalric.maillard@inria.fr}\\
   \addr Inria Lille -- Nord Europe}

\editors{Wee Sun Lee and Taiji Suzuki}

\begin{document}

\maketitle

\begin{abstract}
Leveraging an equivalence property in the state-space of a Markov Decision Process (MDP) has been investigated in several studies. This paper studies equivalence structure  in the reinforcement learning (RL) setup, where transition distributions are no longer assumed to be known. We present a notion of similarity between transition probabilities of various state-action pairs of an MDP, which naturally defines an equivalence structure in the state-action space. We present equivalence-aware confidence sets for the case where the learner knows the underlying structure in advance. These sets are provably smaller than their corresponding equivalence-oblivious counterparts. In the more challenging case of an unknown equivalence structure, we present an algorithm called \ClusteringAlgo\ that seeks to find an (approximate) equivalence structure, and define confidence sets using the approximate equivalence. To illustrate the efficacy of the presented confidence sets, we present \CUCRL, as a natural modification of \UCRL\ for RL in undiscounted MDPs. In the case of a known equivalence structure, we show that \CUCRL\ improves over \UCRL\ in terms of \emph{regret} by a factor of $\sqrt{SA/C}$, in any communicating  MDP with $S$ states, $A$ actions, and $C$ classes, which corresponds to a massive improvement when $C\ll SA$. To the best of our knowledge, this is the first work providing regret bounds for RL when an equivalence structure in the MDP is efficiently exploited. In the case of an unknown equivalence structure, we show through numerical experiments that \CUCRL\ combined with \ClusteringAlgo\ outperforms \UCRL\ in ergodic MDPs.
\end{abstract}

\section{Introduction}\label{sec:Introduction}
This paper studies the Reinforcement Learning (RL) problem, where an agent interacts with an unknown environment in a single stream of observations, with the aim of maximizing the cumulative reward gathered over the course of experience. The environment is modeled as a Markov Decision Process (MDP), with finite state and action spaces, as considered in most literature; we refer to \citep{puterman2014markov,sutton1998reinforcement} for background materials on MDPs, and to Section \ref{sec:equivalence_class}. In order to act optimally or nearly so, the agent needs to learn the parameters of the MDP using the observations from the environment. The agent thus faces a fundamental trade-off between \emph{exploitation vs.~exploration}: Namely, whether to gather more experimental data about the consequences of the actions (exploration) or acting consistently with past observations to maximize the rewards (exploitation); see \citep{sutton1998reinforcement}.
Over the past two decades, a plethora of studies have addressed the above RL problem in the undiscounted setting, where the goal is to minimize the regret (e.g., \citet{bartlett2009regal,jaksch2010near,azar2017minimax}), or in the discounted setting (as in, e.g., \citet{strehl2008analysis}) with the goal of bounding the sample complexity of exploration as defined in \citep{kakade2003phd}. 
In most practical situations, the state-space of the underlying MDP is too large, but often endowed with some \emph{structure}. Directly applying the state-of-the-art RL algorithms, for instance from the above works, and ignoring the structure would lead to a prohibitive regret or sample complexity.


In this paper, we consider RL problems where the state-action space of the underlying MDP exhibits some \emph{equivalence structure}. This is quite typical in many MDPs in various application domains. For instance, in a grid-world MDP when taking action `up' from state $s$ or `right' from state $s'$ when both are away from any wall may result in similar transitions (typically, move towards the target state with some probability, and stay still or transit to other neighbors with the remaining probability); see, e.g., Figure \ref{fig:equivalenceexample} in Section \ref{sec:equivalence_class}. We are interested in exploiting such a structure in order to speed up the learning process. Leveraging an equivalence structure is popular in the MDP literature; see \citep{ravindran2004approximate,li2006towards,abel2016near}. However, most notions are unfortunately not well adapted to the RL setup, that is when the underlying MDP is \emph{unknown}, as opposed to the known MDP setup. In particular, amongst those considering such structures, to our knowledge,  none has provided performance guarantees in terms of regret or sample complexity. Our goal is to find a near-optimal policy, with controlled regret or sample complexity. To this end, we follow a model-based approach, which is popular in the RL literature, and aim at providing a generic model-based approach capable of exploiting this structure, to speed up  learning. We do so by aggregating the information of state-action pairs in the same equivalence class when estimating the transition probabilities or reward function of the MDP.

\paragraph{Contributions.} We make the following contributions. (i) We first introduce a notion of similarity between state-action pairs, which  naturally yields a \emph{partition} of the state-action space $\cS\times \cA$, and induces an equivalence structure in the MDP (see Definition \ref{def:sim}--\ref{def:equi_class}). To our knowledge, while other notions of equivalence have been introduced, our proposed definition appears to be the first, in a discrete RL setup, explicitly using profile (ordering) of distributions. (ii) We present confidence sets that incorporate equivalence structure of transition probabilities and reward function into their definition, when the learner has access to such information. These confidence sets are smaller than those obtained by ignoring equivalence structures. (iii) In the case of an unknown equivalence structure, we present \ClusteringAlgo, which uses confidence bounds of various state-action pairs as a proxy to estimate an empirical equivalence structure of the MDP. (iv) Finally, in order to demonstrate the application of the above equivalence-aware confidence sets, we present \CUCRL, which is a natural modification of the \UCRL\ algorithm \citep{jaksch2010near} employing the presented confidence sets. As shown in Theorem \ref{thm:regretKnownC}, when the learner knows the equivalence structure, \CUCRL\ achieves a regret which is smaller than that of \UCRL\ by a factor of $\sqrt{SA/C}$, where $C$ is the number of classes. This corresponds to a massive improvement when $C\ll SA$. We also verify, through numerical experiments, the superiority of \CUCRL\ over \UCRL\ in the case of an unknown equivalence structure.

\paragraph{Related Work.}
There is a rich literature on state-abstraction (or state-aggregation) in MDPs; we refer to  \citep{li2006towards} on earlier methods, and to \citep{abel2016near} for a good survey of recent approaches. \citep{ravindran2004approximate} introduces aggregation based on homo-morphisms of the model, but with no algorithm nor regret analysis. \citep{dean1997model,givan2003equivalence} consider a partition of state-space of MDPs based on the notion of \emph{stochastic bi-simulation}, which is a generalization of the notion of bi-simulation from the theory of concurrent processes to stochastic processes. This path is  further followed in \citep{ferns2004metrics,ferns2011bisimulation}, where \emph{bi-simulation metrics} for capturing similarities are presented. Bi-simulation metrics can be thought of as quantitative analogues of the equivalence relations, and suggest to resort to optimal transport, which is intimately linked with our notions of similarity and equivalence (see Definition \ref{def:sim}).  
However, these powerful metrics have only been studied in the context of a \emph{known} MDP, and not the RL setup.
The approach in \citep{anand2015asap} is similar to our work in that it considers state-action equivalence. Unlike the present paper, however, it does not consider orderings, transition estimation errors, or regret analysis. Another relevant work to our approach is \citep{ortner2013adaptive} on aggregation of states (but not of pairs, and with no ordering) based on concentration inequalities, a path that we follow. We also mention the works \citep{brunskill2013sample,mandel2016efficient}, where clustering of the state-space is studied. As other relevant works, we refer to  \citep{leffler2007efficient}, where \emph{relocatable action model} is introduced, and to \citep{diuk2009adaptive} that studies RL in the simpler setting of factored MDPs. We also mention interesting works revolving around complementary RL questions including the one on selection amongst different state representations in  \citep{ortner2014selecting} and on state-aliasing in \citep{hallak2013model}.

As part of this paper is devoted to presenting an equivalence structure aware variant of \UCRL, we provide here a brief review of the literature related to \emph{undiscounted}  RL. Undiscounted RL dates back at least to \citep{burnetas1997optimal}, and is thoroughly investigated later on in \citep{jaksch2010near}. The latter work presents \UCRL, which is inspired by multi-armed bandit algorithms. Several studies continued this line,  including \citep{bartlett2009regal,maillard2014hard,azar2017minimax,dann2017unifying,talebi2018variance,fruit2018efficient}, to name a few. Most of these works present \UCRL-style algorithms, and try to reduce the regret dependency on the number of states, as in, e.g.,
 \citep{azar2017minimax,dann2017unifying} (restricted to the episodic RL with a fixed and known horizon). Although the concept of equivalence is well-studied in MDPs, no work seems to have investigated the possibility of defining an aggregation that both is based on state-action pairs (instead of states only) for RL problems, and uses optimal transportation maps combined with
statistical tests. Especially, the use of profile maps seems novel and we show it is also effective.

\section{Model and Equivalence Classes}
\label{sec:equivalence_class}

\subsection{The RL Problem}
In this section, we describe the RL problem, which we study in this paper. Let $\cM = (\cS,\cA,p,\nu)$ be an undiscounted MDP\footnote{Our results can be extended to the discounted case as well.}, where $\cS$ denotes the discrete state-space with cardinality $S$, and $\cA$ denotes the discrete action-space with cardinality $A$. Here, $p$ represents the transition kernel such that $p(s'|s,a)$ denotes the probability of transiting to state $s'$, starting from state $s$ and executing action $a$. Finally,  $\nu$ is a reward distribution function on $[0,1]$, whose mean is denoted by $\mu$.

The game proceeds as follows. The learner starts in some state $s_1\in \cS$ at time $t = 1$. At each time step $t\in \mathbb N$, the learner chooses one action $a\in \cA$ in its current state $s_t$ based on its past decisions and observations. When executing action $a_t$ in state $s_t$, the learner receives a random reward $r_t:=r_t(s_t,a_t)$ drawn independently from distribution $\nu(s_t, a_t)$, and whose mean is $\mu(s_t,a_t)$. The state then transits to a next state $s_{t+1}\sim p(\cdot|s_t,a_t)$, and a new decision step begins.  We refer to \citep{sutton1998reinforcement,puterman2014markov} for background material on MDPs and RL. The goal of the learner is to maximize the \textit{cumulative reward} gathered in the course of interaction with the environment. As  $p$ and $\nu$ are unknown, the learner has to learn them by trying different actions and recording the realized rewards and state transitions.
The performance of the learner can be assessed through the notion of \emph{regret}\footnote{We note that in the discounted setting, the quality of a learning algorithm is usually assessed through the notion of \emph{sample complexity} as defined in \citep{kakade2003phd}.} with respect to an optimal oracle, being aware of $p$ and $\nu$. More formally, as in \citep{jaksch2010near}, under a learning algorithm $\bA$, we define the $T$-step regret as
\als{
\kR(\bA,T) := Tg_\star - \sum_{t=1}^T r_t(s_t,a_t)\,,
}
where $g_\star$ denotes the \emph{average reward} (or \emph{gain}\footnote{See, e.g., \citep{puterman2014markov} for background material on MDPs.}) attained by an optimal policy, and where $a_t$ is chosen by $\bA$ as a function of $((s_{t'},a_{t'})_{t'<t}, s_t)$. Alternatively, the objective of the learner is to minimize the regret, which calls for balancing between exploration and exploitation. In the present work, we are interested in exploiting \textit{equivalence structure} in the state-action space in order to speed up exploration, which, in turn, reduces the regret.



\subsection{Similarity and Equivalence Classes}
We now present a precise definition of the equivalence structure considered in this paper. We first introduce a notion of similarity between state-action pairs of the MDP:

\begin{definition}[Similar state-action pairs]\label{def:sim}
	The pair $(s',a')$ is said to be $\epsilon$-similar to the pair $(s,a)$, for
	$\epsilon=(\epsilon_p,\epsilon_\mu)\!\in\!\Real_+^2$,  if
	\als{	
	\qquad\| p(\sigma_{s,a}(\cdot)|s,a) -p(\sigma_{s',a'}(\cdot)|s',a')\|_1 \leq \epsilon_p\,\quad	\text{and}\quad |\mu(s,a)-\mu(s',a')| \leq \epsilon_{\mu}\,,
	}
	where $\sigma_{s,a} : \{1,\dots,S\} \to \cS$ indexes a permutation of states such that $p(\sigma_{s,a}(1) |s,a) \geq p(\sigma_{s,a}(2) |s,a) \geq \dots \geq p(\sigma_{s,a}(S) |s,a)$. We refer to $\sigma_{s,a}$ as a \textbf{\emph{profile mapping}} (or for short, \emph{\textbf{profile}}) for $(s,a)$, and denote by $\boldsymbol \sigma = (\sigma_{s,a})_{s,a}$ the set of profile mappings of all pairs.
\end{definition}

The notion of similarity introduced above naturally yields a \emph{partition} of the state-action space $\cS\times \cA$, as detailed in the following definition:
\vspace{-2mm}

\begin{definition}[Equivalence classes]\label{def:equi_class}
$(0,0)$-similarity is an equivalence relation and induces a canonical partition of $\cS\times\cA$. We refer to such a canonical partition as \textbf{\emph{equivalence classes}} or \textbf{\emph{equivalence structure}}, denote it by $\cC$, and let $C:=|\cC|$.
\end{definition}
\vspace{-3mm}

In order to help understand Definitions \ref{def:sim} and \ref{def:equi_class}, we present in Figure~\ref{fig:equivalenceexample} an MDP  with 13 states, where the state-action pairs (6,\textsf{Up}) and (8,\textsf{Right}) are equivalent up to a permutation: Let the permutation $\sigma$ be such that $\sigma(2)=9$, $\sigma(6)=8$, and $\sigma(i)=i$ for all $i\neq 2,6$. Now $p(\sigma(x)|6,\textsf{Up})=p(x|8,\textsf{Right})$ for all $x\in \cS$, and thus, the pairs (8,\textsf{Right}) and (6,\textsf{Up}) belong to the same class.

\vspace{-1mm}
\begin{figure}[h]
	\begin{minipage}[r]{0.5\textwidth}
    \begin{center}
		\includegraphics[trim={0mm, 5mm, 0mm, 2mm},clip, width=\textwidth]{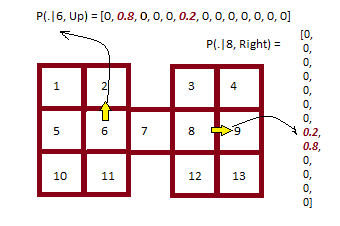}
    \end{center}
	\end{minipage}\hfill
	\begin{minipage}[l]{0.5\textwidth}	
        \begin{center}
		\caption{A grid-world MDP showing similar transitions from state-action pairs (6,\textsf{Up}) and (8,\textsf{Right}).}
		\label{fig:equivalenceexample}
        \end{center}
	\end{minipage}
\end{figure}
\vspace{-3mm}

\begin{remark}
	Crucially, the equivalence relation is not only stated about states, but about \emph{state-action pairs}.
	For instance, pairs (6,\textsf{Up}) and (8,\textsf{Right})  in this example are in the same class although corresponding to playing different actions in different states.
\end{remark}
\vspace{-3mm}

\vspace{-3mm}
\begin{remark}
The profile mapping $\sigma_{s,a}$ in Definition \ref{def:sim} may not be unique in general, especially if distributions have sparse supports. For ease of presentation, in the sequel we assume that the restriction of $\sigma_{s,a}$ to the support of $p(\cdot|s,a)$ is uniquely defined. We also remark that Definition~\ref{def:sim} can be easily generalized by replacing the $\|\cdot\|_1$ norm with other contrasts, such as the KL divergence, squared distance, etc.
\end{remark}
\vspace{-3mm}

In many environments considered in RL with large state and action spaces, the number $C$ of equivalent classes of state-action pairs using Definitions~\ref{def:sim}--\ref{def:equi_class} stays small even for large $SA$, thanks to the profile mappings. This is the case in typical grid-world MDPs as well as in \emph{RiverSwim} shown in Figure \ref{fig:river_swim}. For example, in \emph{Ergodic RiverSwim} with $L$ states, we have $C=6$. We also refer to Appendix \ref{app:examples} for additional illustrations of grid-world MDPs. This remarkable feature suggests that leveraging this structure may yield significant speed-up in terms of learning guarantees if well-exploited.

\vspace{-1mm}
\section{Equivalence-Aware Confidence Sets}
\label{sec:confidence_bounds}
We are now ready to present an approach that defines confidence sets for $p$ and $\mu$ taking into account the equivalence structure in the MDP. The use of confidence bounds in a model-based approach is related to strategies implementing the \emph{optimism in the face of uncertainty} principle, as in stochastic bandit problems \citep{lai1985asymptotically,auer2002finite}. Such an approach relies on maintaining a set of plausible MDPs (models) that are consistent with the observations gathered, and where the set contains the true MDP with high probability. Exploiting equivalence structure of the MDP, one could obtain a more precise estimation of mean reward $\mu$ and transition kernel $p$ of the MDP by \emph{aggregating} observations from various state-action pairs in the same class. This, in turn, yields \emph{smaller} (hence, better) sets of models.

\paragraph{Notations.}We introduce some necessary notations. Under a given algorithm, for a pair $(s,a)$, we denote by $N_t(s,a)$ the total number of observations of $(s,a)$ up to time $t$. Let us define $\widehat \mu_{t}(s,a)$ as the empirical mean reward built using $N_t(s,a)$ i.i.d.~samples from $\nu(s,a)$, and $\widehat p_t(\cdot|s,a)$ as the empirical distribution built using $N_t(s,a)$ i.i.d.~observations from $p(\cdot|s,a)$. For a set $c\subseteq \cS\times \cA$, we denote by $n_t(c)$ the total number of observations of pairs in $c$ up to time $t$, that is $n_t(c) := \sum_{(s,a)\in c} N_t(s,a)$. For $c \subseteq \cS\times \cA$, we further denote by $\widehat \mu_{t}(c)$ and $\widehat p_t(\cdot|c)$ the empirical mean reward and transition probability built using $n_t(c)$ samples, respectively; we provide precise definitions of $\widehat\mu_t(c)$ and $\widehat p_t(\cdot|c)$ later on in this section.

\begin{figure}[t]
\centering
\tiny
\def\svgwidth{0.7\columnwidth}
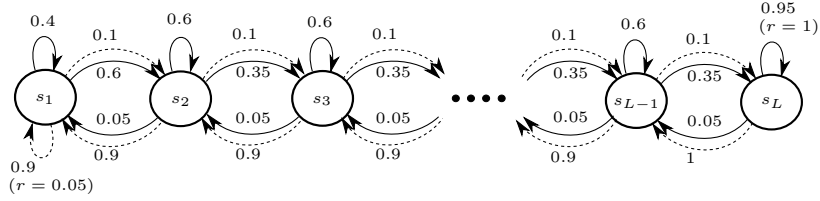
\caption{The $L$-state \emph{Ergodic RiverSwim} MDP}\label{fig:river_swim}
\end{figure}


For a given confidence parameter $\delta$ and time $t$, we write $\cM_{t,\delta}$ to denote the set of plausible MDPs at time $t$, which may be generically expressed as
\begin{equation}\label{eq:model_set_UCRL2}
\cM_{t,\delta} = \{(\cS,\cA,p',\nu'): p'\in \textsf{CB}_{t,\delta} \hbox{ and } \mu'\in \textsf{CB}'_{t,\delta} \}\,,
\end{equation}
where $\textsf{CB}_{t,\delta}$ (resp.~$\textsf{CB}'_{t,\delta}$) denotes the confidence set for $p$ (resp.~$\mu$) centered at $\widehat p$ (resp.~$\widehat \mu$), and where $\mu'$ is the mean of $\nu'$. Note that both $\textsf{CB}_{t,\delta}$ and $\textsf{CB}'_{t,\delta}$ depend on $N_t(s,a), (s,a)\in \cS\times \cA$. For ease of presentation, in the sequel we consider the following confidence sets\footnote{The approach presented in this section can be extended to other concentration inequalities, as well.} used in several model-based RL algorithms, e.g., \citep{jaksch2010near,dann2017unifying}: 
\begin{align}
\label{eq:L1_CB}
\textsf{CB}_{t,\delta}&:= \big\{p': \|\widehat p_{t}(\cdot|s,a)  - p'(\cdot|s,a) \|_1 \leq \bW_{N_t(s,a)}\big(\tfrac{\delta}{SA}\big), \, \forall s,a\big\}\, , \\
\textsf{CB}'_{t,\delta}&:= \big\{\mu':|\widehat \mu_t(s,a)  - \mu'(\cdot|s,a)| \leq \bH_{N_t(s,a)}\big(\tfrac{\delta}{SA}\big), \, \forall s,a\big\}\, , \qquad \hbox{ where} \nonumber
\end{align}
\vspace{-3mm}
\vspace{-3mm}
\begin{align}\label{eq:beta_n}
\bW_n(\delta)\!=\!\sqrt{\frac{2(1+\frac{1}{n})\log\big(\sqrt{n+1}\frac{2^S-2}{\delta}\big)}{n} }
\quad \text{and} \quad  \bH_n(\delta)=\sqrt{ \frac{(1+\frac{1}{n})\log(\sqrt{n+1}/\delta)}{2n}}\,,\,\,\forall n.
\end{align}
These confidence sets were derived by combining Hoeffding's \citep{hoeffding1963probability} and Weissman's \citep{weissman2003inequalities} concentration inequalities with the Laplace method \citep{pena2008self,abbasi2011improved}, which enables to handle the random stopping times $N_t(s,a)$ in a sharp way; we refer to \citep{maillard2019mathematics} for further discussion.
In particular, this ensures that the true transition function $p$ and mean reward function $\mu$ are contained in the confidence sets with probability at least $1-2\delta$, uniformly over all time $t$.

\vspace{-3mm}
\begin{remark}
As the bounds for $\mu$ and $p$ are similar, to simplify the presentation, from now on  we assume the mean reward function $\mu$ is known\footnote{This is a common assumption in the RL literature; see, e.g., \citep{bartlett2009regal}.}.
\end{remark}
\vspace{-3mm}

We now provide modifications to $\textsf{CB}_{t,\delta}$ in order to exploit the equivalence structure $\cC$, when the learner knows $\cC$ in advance. The case of an unknown $\cC$ is addressed in Section \ref{sec:clustering}. 

\subsection{Case 1: Known Classes and Profiles}
Assume that an oracle provides the learner with a perfect knowledge of the equivalence classes $\cC$ as well as profiles $\boldsymbol{\sigma} = (\sigma_{s,a})_{s,a}$. In this ideal situation, the knowledge of $\cC$ and $\boldsymbol\sigma$ allows to straightforwardly aggregate observations from all state-action pairs in the same class to build more accurate estimates of $p$ and $\mu$. Formally, for a class $c\subset \cC$, we define
\begin{align}\label{eq:p_hat_C_sigma}
\widehat p^{\boldsymbol{\sigma}}_{t} ( x |c)  &=\frac{1}{n_t(c)}\sum_{(s,a) \in c}  N_t(s,a) \widehat p_{t}\big(\sigma_{s,a}(x)|s,a\big)\,,\qquad \forall x\in \cS,
\end{align}
where we recall that $n_t(c)= \sum_{(s,a)\in c} N_t(s,a)$. The superscript $\boldsymbol\sigma$ in (\ref{eq:p_hat_C_sigma})  signifies that the aggregate empirical distribution $\widehat p^{\boldsymbol{\sigma}}_{t}$ depends on $\boldsymbol\sigma$.
Having defined $\widehat p^{\boldsymbol{\sigma}}_{t}$, we modify the confidence set (\ref{eq:L1_CB}) by modifying the $L_1$ bound there as follows:
\begin{align}\label{eq:L1_CB_C_sigma}
\| \widehat p^{\boldsymbol\sigma}_{t}(\boldsymbol{\sigma}^{-1}(\cdot)|c) - p'(\cdot|c) \|_1 \leq \beta_{n_t(c)}\big(\tfrac{\delta}{C}\big)\,, \quad \forall c\in \cC,
\end{align}
and further define: $\textsf{CB}_{t,\delta} (\cC,\boldsymbol\sigma):= \big\{p'\!: \text{(\ref{eq:L1_CB_C_sigma}) holds}\big\}$, where $(\cC,\boldsymbol\sigma)$ stresses that $\cC$ and $\boldsymbol\sigma$ are provided as input. Then, for all time $t$ and class $c\in \cC$, by construction, the true transition $p$ belongs to $\textsf{CB}_{t,\delta} (\cC,\boldsymbol\sigma)$, with probability greater than $1-\delta$.




\vspace{-1mm}
\begin{remark}
	It is crucial to remark that the above confidence set does not use elements of $\cC$  as ``meta states'' (i.e., replacing the states with classes), as considered for instance in the literature on state-aggregation. Rather, the classes are only used to \emph{group} observations from different sources and build more refined estimates for each pair: The plausible MDPs are built using the same state-space $\cS$ and action-space $\cA$, unlike in, e.g., \citep{ortner2013adaptive}.
\end{remark}
\vspace{-3mm}

\subsection{Case 2: Known Classes, Unknown Profiles}
Now we consider a more realistic setting when the oracle provides $\cC$ to the learner, but $\boldsymbol\sigma$ is unknown. In this more challenging situation, we need to \emph{estimate} profiles as well. Given time $t$, we find an \emph{empirical profile mapping} (or for short, empirical profile) $\sigma_{s,a,t}$ satisfying
\als{
\widehat p_{t}(\sigma_{s,a,t}(1) |s,a) \geq \widehat p_t(\sigma_{s,a,t}(2) |s,a) \geq \dots \geq \widehat p_t(\sigma_{s,a,t}(S) |s,a)\, ,
}
and define $\boldsymbol\sigma_t=(\sigma_{s,a,t})_{s,a}$. We then build the modified empirical estimate in a similar fashion to (\ref{eq:p_hat_C_sigma}): For any $c\in \cC$,
\als{
\widehat p^{\boldsymbol\sigma_t}_t ( x |c)  &=\frac{1}{n_t(c)}\sum_{(s,a) \in c}  N_t(s,a) \widehat p_t(\sigma_{s,a,t}(x)|s,a)\,,\quad \forall x\in \cS.
}
Then, we may modify the $L_1$ inequality in (\ref{eq:L1_CB}) as follows:
\begin{align}\label{eq:L1_CB_C}
\| \widehat p^{\boldsymbol\sigma_t}_t(\boldsymbol{\sigma}_t^{-1}(\cdot)|c) - p'(\cdot|c) \|_1 \leq \frac{1}{n_t(c)}\sum_{(s,a)\in c} N_t(s,a)\beta_{N_t(s,a)}\big(\tfrac{\delta}{C}\big)\, , \quad \forall c\in \cC,
\end{align}
which further yields the following modified confidence set that uses only $\cC$ as input:
$
\textsf{CB}_{t,\delta} (\cC):= \big\{p'\!: \text{(\ref{eq:L1_CB_C}) holds}\big\}$.
The above construction is justified by the following non-expansive property of the ordering operator, as it ensures that Weissman's concentration inequality also applies to the ordered empirical distribution:


\begin{lemma}[Non-expansive ordering]\label{lem:ordered}
	Let $p$ and $q$ be two discrete distributions, defined on the same alphabet $\cS$, with respective profile mappings $\sigma_p$ and $\sigma_q$. Then,
	\beqan
	\|p(\sigma_p(\cdot))- q(\sigma_q(\cdot))\|_1 \leq
	\| p- q\|_1\,.
	\eeqan
\end{lemma}
\vspace{-2mm}

The proof of Lemma \ref{lem:ordered} is provided in Appendix \ref{app:ordered_proof}. An immediate corollary follows.
\vspace{-1mm}

\begin{corollary}
The confidence set $\textsf{CB}_{t,\delta} (\cC)$ contains the true transition function $p$ with probability at least $1-\delta$, uniformly over all time $t$.
\end{corollary}
\vspace{-3mm}

\section{Unknown Classes: The \ClusteringAlgo\ Algorithm}
\label{sec:clustering}
In this section, we turn to the most challenging situation when both $\cC$ and $\boldsymbol \sigma$ are unknown to the learner. To this end, we first introduce an algorithm, which we call  \ClusteringAlgo, that finds an approximate equivalence structure in the MDP by grouping transition probabilities based on statistical tests. \ClusteringAlgo\ is inspired by \citep{khaleghi2016consistent} that provides a method for clustering time series. Interestingly enough, \ClusteringAlgo\ does not require the knowledge of the number of classes in advance.

We first introduce some definitions. Given $u,v\subseteq \cS\times \cA$, we define the distance between $u$ and $v$ as
$
d(u,v):=  \|p^{\boldsymbol \sigma_u}(\cdot|u) - p^{\boldsymbol\sigma_v}(\cdot|v)\|_1
$.
\ClusteringAlgo\ relies on finding subsets of $\cS\times \cA$ that are \emph{statistically close} in terms of the distance function $d(\cdot,\cdot)$. As $d(\cdot,\cdot)$ is unknown, \ClusteringAlgo\ relies on a lower confidence bound on it:  For $u,v\subseteq \cS\times \cA$, we define the \emph{lower-confidence distance function} between $u$ and $v$ as
\als{
\widehat d(u,v):= \widehat d_{t,\delta}(u,v):= \big\|\widehat p^{\boldsymbol\sigma_{u,t}}_t(\cdot|u) - \widehat p_t^{\boldsymbol\sigma_{v,t}}(\cdot|v)\big\|_1 - \epsilon_{u,t} - \epsilon_{v,t}\, ,
}
where for $u\in \cS\times \cA$ and $t\in \Nat$, we define $\epsilon_{u,t}:=\frac{1}{n_t(u)}\sum_{\ell\in u} N_t(\ell)\bW_{N_t(\ell)}\big(\tfrac{\delta}{SA}\big)$.
We stress that, unlike $d(\cdot,\cdot)$, $\widehat d(\cdot,\cdot)$ is not a distance function.
\vspace{-1mm}




\begin{definition}[PAC Neighbor]
For a given equivalence structure $\cC$, and given $c\in \cC$, we say that $c'\in \cC$ is a \emph{PAC Neighbor} of $c$ if it satisfies:
(i) $\widehat d(c,c')\leq 0$; (ii) $\widehat d(\{j\},\{j'\})\leq 0$, for all $j\in c$ and $j'\in c'$; and (iii) $\widehat d(\{j\}, c\cup c') \leq 0$, for all $j\in c\cup c'$.
We further define $\cN(c) := \big\{c'\in\cC\setminus\{c\}\!: \text{(i)--(iii) hold} \big\}$ as the set of all PAC Neighbors of $c$.
\end{definition}
\vspace{-3mm}

\begin{definition}[PAC Nearest Neighbor]
\label{def:PAC_NN}
For a given equivalence structure $\cC$ and $c\in\cC$, we define the \emph{PAC  Nearest Neighbor} of $c$ (when it exists) as:
\als{
	\emph{\textsf{Near}}(c,\cC) \in \argmin_{u\in \cN(c)} \widehat d(c,u)\, .
}
\end{definition}
\vspace{-3mm}


\ClusteringAlgo\ proceeds as follows. At time $t$, it receives as input a parameter $\alpha>1$ that controls the level of aggregation, as well as $N_t(s,a)$ for all pairs $(s,a)$.
Starting from the trivial partition of $\{1,\ldots,SA\}$ into $\cC^0:=\big\{\{1\},\dots,\{SA\}\big\}$, the algorithm builds a coarser partition by iteratively merging elements of $\cC^0$ that are \emph{statistically close}. More precisely, the algorithm sorts elements of $\cC^0$ in a non-increasing order of $n_t(c),c\in \cC^0$ so as to promote pairs with the tightest confidence intervals. Then, starting from $c$ with the largest $n_t(c)$, it finds  the PAC Nearest Neighbor  $c'$ of $c$, that is $c'=\textsf{Near}(c,\cC^0)$. If $\frac{1}{\alpha}\leq \frac{n_t(c)/L(c)}{n_t(c')/L(c')} \leq \alpha$, where $L(c)=|c|$, the algorithm merges $c$ and $c'$, thus leading to a novel partition $\cC^1$, which contains the new cluster $c\cup c'$, and removes $c$ and $c'$. The algorithm continues this procedure with the next set in $\cC^0$, until exhaustion, thus finishing the creation of the novel partition $\cC^1$ of $\{1,\ldots,SA\}$.
\ClusteringAlgo\ continues this process, by ordering the elements of $\cC^1$ in a non-increasing order, and carrying out similar steps as before, yielding the new partition $\cC^2$. \ClusteringAlgo\ continues the same procedure until  iteration $k$ when $\cC^{k+1} = \cC^k$ (convergence). The pseudo-code of \ClusteringAlgo\ is shown in Algorithm \ref{alg:ConfidentClustering_group}.


\begin{algorithm}[H]
		\begin{algorithmic}[-1]
			\footnotesize
            \REQUIRE $N_t$, $\alpha$
			\STATE \textbf{Initialization:}
			$\cC^0\leftarrow \{\{1\},\{2\},\ldots,\{SA\}\}$; 
\; $n \leftarrow N_t$; \; $L \leftarrow \mathbf 1_{SA}$
			\STATE $\texttt{changed} \leftarrow \textsf{True}$; \; $k\leftarrow 1$;
			\WHILE{\texttt{changed}}
             \STATE $\cC^{k+1} \leftarrow \cC^k$;
			 \STATE $\texttt{changed} \leftarrow \textsf{False}$;
			 \STATE $\texttt{Index} \leftarrow \mathrm{argsort}(n)$;
    	    	\FORALL{$i\in \texttt{Index}$}
                    \IF{$n(i)= 0$}
                        \STATE Break;
                    \ENDIF
                    \IF{$\textsf{Near}(i, \cC^{k-1})\neq \emptyset$}
                            \STATE $j\leftarrow \textsf{Near}(i, \cC^{k-1})$;
                            \IF{$\frac{1}{\alpha}\leq \frac{n(i)/L(i)}{n(j)/L(j)} \leq \alpha$}
                                \STATE $\widehat p_t^{\boldsymbol\sigma_{j,t}}(\cdot|j) \leftarrow \frac{1}{n(j) + n(i)}\Big(n(j)\widehat p_t^{\boldsymbol\sigma_{j,t}}(\cdot|j)  + n(i)\widehat p_t^{\boldsymbol\sigma_{i,t}}(\cdot|i) \Big)$
                                \STATE $L(i) \leftarrow L(j) + L(i)$; \,\,  $n(i) \leftarrow n(j) + n(i)$;
                                \STATE $n(j) \leftarrow 0$, \,\, $L(j) \leftarrow 0$;
                                \STATE $\cC^{k+1} \leftarrow \cC^{k+1}\setminus (\{i\},\{j\}) \cup \{i,j\} $;
            			        \STATE $\texttt{changed} \leftarrow \textsf{True}$;
                           \ENDIF
                    \ENDIF
    			\ENDFOR
              \STATE $k\leftarrow k+1$;
			\ENDWHILE
        \OUTPUT $\cC^k$
		\end{algorithmic}
	\caption{\ClusteringAlgo}
	\label{alg:ConfidentClustering_group}
\end{algorithm}
\normalsize

The purpose of condition $\frac{1}{\alpha}\leq \frac{n_t(c)/L(c)}{n_t(c')/L(c')} \leq \alpha$ is to ensure the stability of the algorithm. It prevents merging pairs whose numbers of samples differ a lot. We note that a very similar condition (with $\alpha =2$) is considered in \citep{ortner2013adaptive} for state-aggregation. Nonetheless, we believe such a condition could be relaxed.
\vspace{-3mm}

\begin{remark}
Since at each iteration, either two or more subsets are merged, \ClusteringAlgo\ converges after, at most, $SA - 1$ steps.
\end{remark}
\vspace{-3mm}

We provide a theoretical guarantee for the correctness of \ClusteringAlgo\ for the case when $\alpha$ tends to infinity. The result relies on the following separability assumption:
\begin{assumption}[Separability]\label{ass:separation} There exists some $\Delta>0$  such that
\begin{align*}
\forall c\neq c'\in\cC,\, & \forall \ell\in c, \forall\ell'\in c',\quad d(\{\ell\},\{\ell'\})\geq \Delta\, .
\end{align*}
\end{assumption}


\begin{proposition}\label{prop:clustering_guarantee}
Under Assumption~\ref{ass:separation}, provided that $\min_{s,a}N_t(s,a)\!>\! f^{-1}(\Delta)$,
where $f: n\!\mapsto\! 4\bW_n(\tfrac{\delta}{SA})$, \ClusteringAlgo\ with the choice $\alpha\to\infty$ outputs the correct equivalence structure $\cC$ of state-action pairs with probability at least $1-\delta$.
\end{proposition}
\vspace{-3mm}

The proof of Proposition \ref{prop:clustering_guarantee} is provided in Appendix \ref{app:clustering}. We note that
Assumption \ref{ass:separation} bears some similarity to the separability assumption used in \citep{brunskill2013sample}. Note further although the proposition relies on Assumption \ref{ass:separation}, we believe one may be able to derive a similar result under a weaker assumption as well. We leave this for future work.

Now we turn to defining  the aggregated confidence sets. Given $t$, let $\cC_t$ denote the equivalence structure output by the algorithm. 
We may use the following confidence set:
\begin{align}
\label{eq:L1_CB_C_empirical}
\textsf{CB}_{t,\delta}(\cC_t) :=\Big\{p':\| \widehat p^{\boldsymbol\sigma_t}_t(\boldsymbol\sigma_{\boldsymbol t}^{-1}(\cdot)|c) - p'(\cdot|c) \|_1 \!\leq \! \sum\nolimits_{(s,a)\in c} \tfrac{N_t(s,a)}{n_t(c)}\beta_{N_t(s,a)}\big(\tfrac{\delta}{SA}\big)
\,, \,\, \forall c\in \cC_t\Big\}\, .
\end{align}

\section{Application: The \CUCRL\ Algorithm}

This section is devoted to presenting some  applications of equivalence-aware confidence sets introduced in Section \ref{sec:confidence_bounds}. We present \CUCRL, a natural modification of \UCRL\ \citep{jaksch2010near}, which is capable of exploiting the equivalence structure of the MDP.
We consider variants of \CUCRL\ depending on which information is available to the learner in advance.

First, we briefly recall \UCRL.
At a high level, \UCRL\ maintains the set $\cM_{t,\delta}$ of MDPs at time $t$,\footnote{This set is described by the Weismann confidence bounds combined with the Laplace method. The original \UCRL\ algorithm in \citep{jaksch2010near} uses looser confidence bounds relying on union bounds instead of the Laplace method.} which is defined in (\ref{eq:model_set_UCRL2}). It then implements the optimistic principle by trying to compute
$\overline{\pi}_t^+ = \argmax_{\pi:\cS\to\cA} \max_{M \in \cM_{t,\delta}} g_\pi^M$, where $g_\pi^M$ denotes the gain of policy $\pi$ in MDP $M$. This is carried out approximately by the \texttt{Extended Value Iteration (EVI)} algorithm that builds a near-optimal policy $\pi^+_t$ and MDP $\widetilde M_t$ such that $g_{\pi^+_t}^{\widetilde M_t}  \geq \max_{\pi,  M\in\cM_{t,\delta} }g_\pi^M - \tfrac{1}{\sqrt{t}}$. Finally,  \UCRL\ does not recompute $\pi^+_t$ at each time step. Instead, it proceeds in internal episodes (indexed by  $k\in\Nat$), and computes $\pi_t^+$ only at the starting time $t_k$ of each episode, defined as $t_1=1$ and for all $k>1$,
$
t_k  \!=\! \min\!\Big\{ t \!>\! t_{k-1}\!: \exists s,a,  \nu_{t_{k-1}:t}(s,a)\!\geq\! N_{t_{k-1}}(s,a)^+\!\Big\},
$
where $\nu_{t_1:t_2}(s,a)$ denotes the number of observations of pair $(s,a)$ between time $t_1\!+\!1$ and $t_2$, and where for $z\in \Nat$, $z^+\!:=\!\max\{z,1\}$. We provide the pseudo-code of \UCRL\ in Appendix \ref{app:CUCRL2}.


\subsection{\CUCRL: Known Equivalence Structure}\label{sec:kUkM}
Here we assume that the learner knows  $\cC$ and $\boldsymbol\sigma$ in advance, and provide a variant of \UCRL, referred to as \CUCRL$(\cC,\boldsymbol\sigma)$, capable of exploiting the knowledge on $\cC$ and $\boldsymbol\sigma$. Given $\delta$, at time $t$, \CUCRL$(\cC,\boldsymbol\sigma)$ uses the following set of models
\als{
\cM_{t,\delta}(\cC,\boldsymbol\sigma) &= \Big\{ (\cS,\cA, p', \nu):  p'\!\in\!\textsf{Pw}(\cC)\,\,
 \text{ and } \,\,p'_{\cC}\!\in\!\textsf{CB}_{t,\delta}(\cC,\boldsymbol\sigma)
\Big\}\,,
}
where
 $\textsf{Pw}(\cC)$ denotes the state-transition functions that are piece-wise constant on $\cC$, and where
 $p'_{\cC}$ denotes the function induced by $p'\in \textsf{Pw}(\cC)$ over $\cC$ (that is $p'(\cdot|s,a)=p'_\cC(\cdot|c)$ for all $(s,a)\!\in\! c$).
Moreover, \CUCRL$(\cC,\boldsymbol\sigma)$ defines
$$
t_{k+1}  = \min\!\Big\{ t>t_{k}: \exists c\in\cC\!:\sum\nolimits_{(s,a)\in c}\nu_{t_k:t}(s,a)\geq n_{t_k}(c)^+\!\Big\}\,.
$$
We note that forcing the condition $p'\!\in\!\textsf{Pw}(\cC)$ may be computationally difficult. To ensure efficient implementation, we use the same \texttt{EVI} algorithm of \UCRL, where for $(s,a)\in c$, we replace $\widehat p_t(\cdot|s,a)$ and $\beta_{N_t(s,a)}(\tfrac{\delta}{SA})$ respectively with $\widehat p^{\boldsymbol{\sigma}}_{t}(\cdot|c)$ and $\beta_{n_t(c)}(\tfrac{\delta}{C})$. The precise modified steps of  \CUCRL$(\cC,\boldsymbol\sigma)$ are presented in Appendix \ref{app:CUCRL2} for completeness. An easy modification of the analysis of \citep{jaksch2010near} yields:
\begin{theorem}[Regret of \CUCRL$(\cC,\sigma)$]\label{thm:regretKnownC}
	With probability higher than $1-3\delta$,  uniformly over all time horizon $T$,
	\als{
	\kR(\mathrm{\CUCRL}(\cC,\sigma),T) \leq  18\sqrt{CT\big(S + \log(2C\sqrt{T+1}/\delta)\big)} +  DC\log_2(\tfrac{8T}{C})\,.
	}
\end{theorem}

The proof of Theorem \ref{thm:regretKnownC} is provided in Appendix \ref{app:regretbound}. This theorem shows that efficiently exploiting the knowledge of $\cC$ and $\boldsymbol\sigma$ yields an improvement over the regret bound of \UCRL\ by a factor of $\sqrt{SA/C}$, which could be a huge improvement when $C\ll SA$. This is the case in, for instance, many grid-world environments thanks to Definitions \ref{def:sim}--\ref{def:equi_class}; see Appendix \ref{app:examples}. 

\subsection{\CUCRL: Unknown Equivalence Structure}
Now we consider the case where  $\cC$ is unknown to the learner.
In order to accommodate this situation, we use \ClusteringAlgo\ in order to estimate the equivalence structure.

We introduce \CUCRL, which proceeds similarly to \CUCRL$(\cC,\boldsymbol\sigma)$. At each time $t$, \ClusteringAlgo\ outputs $\cC_t$ as an estimate of the true equivalence structure $\cC$. Then, \CUCRL\ uses the following set of models taking $\cC_t$ as input:
\als{
\cM_{t,\delta}(\cC_t) = \bigg\{ (\cS,\cA, p',\nu) :
p'\!\in\!\textsf{Pw}(\cC_t) \,\,\text{ and } \,\, p'_{\cC_t}\!\in\!\textsf{CB}_{t,\delta}(\cC_t)
\bigg\} \, .
}
Further, it sets the starting step of episode $k+1$ as:
\als{
t_{k+1} = \min\bigg\{t > t_{k}: \exists c\!\in\!\cC_{t_k}, \, \sum\nolimits_{(s,a)\in c}\nu_{t_{k}:t}(s,a)\!\geq\! n_{t_k}(c)^+
\text{ or }\
\exists s,a,\,  \nu_{t_{k}:t}(s,a)\!\geq\! N_{t_k}(s,a)^+\bigg\}\,.
}
Similarly to \CUCRL$(\cC,\boldsymbol\sigma)$, we use a modification of \texttt{EVI} to implement \CUCRL.

\begin{remark}\label{rem:CUCRL_rem}
Note that $\cM_{t,\delta}(\cC)\neq \cM_{t,\delta}(\cC_t)$ as we may have $\cC_t\neq \cC$. Nonetheless, the design of \ClusteringAlgo, which relies on confidence bounds, ensures that $\cC_t$ is informative enough, in the sense that $\cM_{t,\delta}(\cC_t)$ could be much smaller (hence, better) than a set of models that one would obtain by ignoring equivalence structure; this is also validated by the numerical experiments in ergodic environments in the next subsection.
\end{remark}

\subsection{Numerical Experiments}\label{sec:xps}
We conduct numerical experiments to examine the performance of the proposed variants of \CUCRL, and compare it to that of \UCRLLaplace\footnote{\UCRLLaplace\ is a variant of \UCRL, which uses confidence bounds derived by combining Hoeffding's and Weissman's inequalities with the Laplace method, as in (\ref{eq:L1_CB}). We stress that \UCRLLaplace\ attains a smaller regret than the original \UCRL\ of  \citep{jaksch2010near}.}. In our experiments, for running \ClusteringAlgo, as a sub-routine of \CUCRL, we set $\epsilon_{u,t}=\beta_{n_t(u)}\big(\tfrac{\delta}{3SA}\big)$ for $u\in \cS\times \cA$ (in the definition of both $\widehat d(\cdot,\cdot)$ and $\textsf{CB}_{t,\delta}(\cC_t)$). Although this could lead to a biased estimation of $p$, such a bias is controlled thanks to using $\alpha>1$ (in our experiments, we set $\alpha=4$).

In the first set of experiments, we examine the regret of various algorithms in ergodic environments. Specifically, we consider the ergodic \emph{RiverSwim} MDP, shown in Figure \ref{fig:river_swim}, with 25 and 50 states. In both cases, we have $C=6$ classes. 
In Figure \ref{fig:regret_ergodic}, we plot the regret against time steps under \CUCRL$(\cC,\boldsymbol\sigma)$, \CUCRL, and \UCRLLaplace\ executed in the aforementioned environments. The results are averaged over 100 independent runs, and the 95\% confidence intervals are shown. All algorithms use $\delta = 0.05$, and for \CUCRL, we use $\alpha=4$. As the curves show, the proposed \CUCRL\ algorithms significantly outperform \UCRLLaplace, and \CUCRL$(\cC,\boldsymbol\sigma)$ attains the smallest regret. In particular, in the 25-state environment and at the final time step, \CUCRL$(\cC,\boldsymbol\sigma)$ attains a regret smaller than that of \UCRLLaplace\ by a factor of approximately $\sqrt{SA/C}=\sqrt{50/6} \approx 2.9$, thus verifying Theorem \ref{thm:regretKnownC}. Similarly, we may expect an improvement in regret by a factor of around $\sqrt{SA/C}=\sqrt{100/6} \approx 4.1$ in the other environment. We however get a better improvement (by a factor of around 8), which can be attributed to the increase in the regret of \UCRLLaplace\ due to a long burn-in phase (i.e., the phase before the algorithm starts learning).

\begin{figure}[t]
\floatconts
  {fig:regret_ergodic}
  {\caption{Regret of various algorithms in Ergodic RiverSwim environments}}
  {
    \subfigure[25-state Ergodic RiverSwim]{\label{fig:circle}%
      \includegraphics[width=0.49\linewidth]{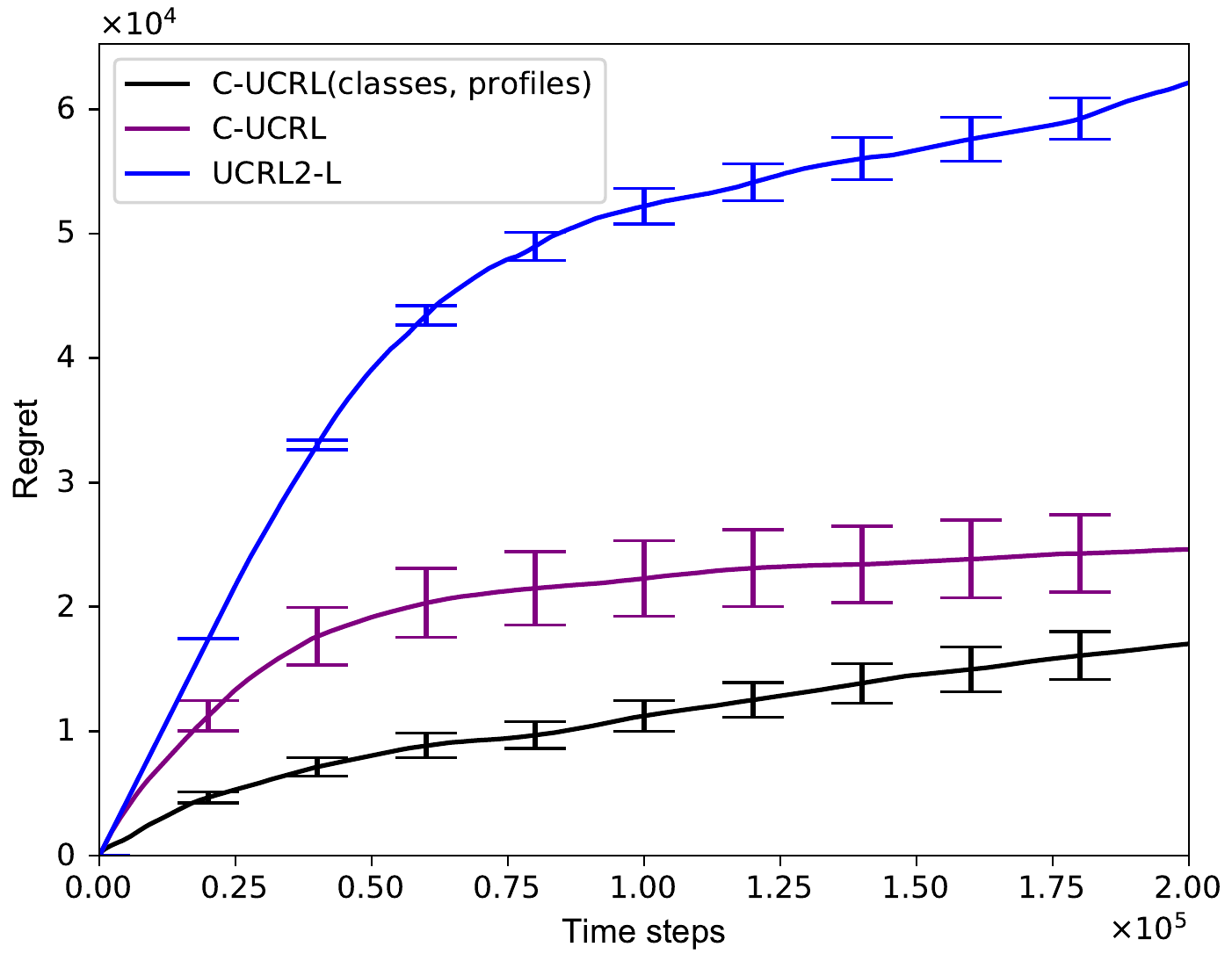}}%
    \subfigure[50-state Ergodic RiverSwim]{\label{fig:square}%
      \includegraphics[width=0.49\linewidth]{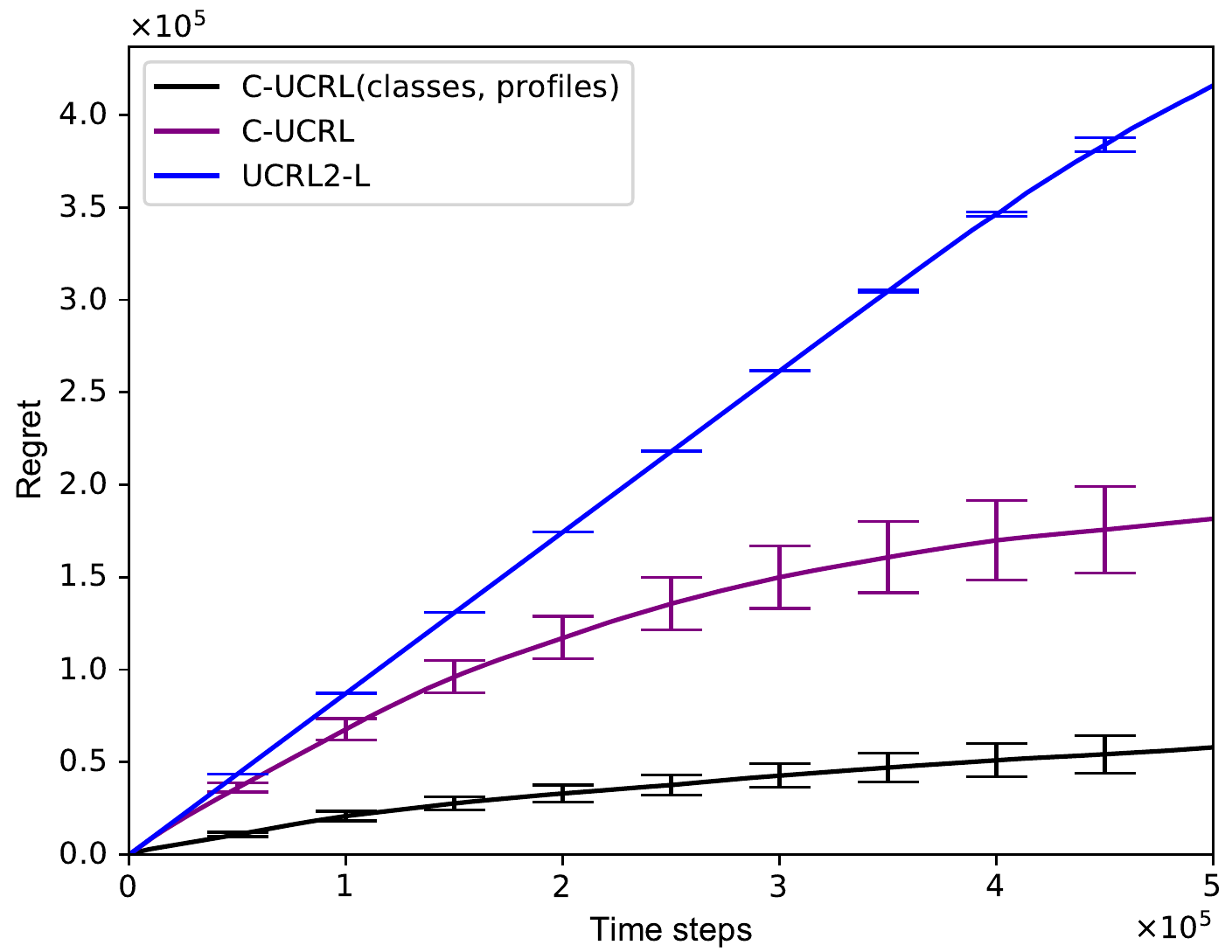}}
  }
\end{figure}

We now turn our attention to the quality of approximate equivalence structure produced by \ClusteringAlgo\ (Algorithm \ref{alg:ConfidentClustering_group}),  which is run as a sub-routine of \CUCRL. To this aim, we introduce two performance measures to assess the quality of clustering: The first one is defined as the total number of pairs that are \emph{mis-clustered}, normalized by the total number $SA$ of pairs. We refer to this measure as the \emph{mis-clustering ratio}. More precisely, let $\cC_t$ denote the empirical equivalence structure output by \ClusteringAlgo\ at time $t$. For a given $c\in \cC_t$, we consider the restriction of $\cC$ to $c$, denoted by $\cC|c$. We find $\ell(c) \in \cC|c$ that has the largest cardinality: $\ell(c)\in \argmax_{x\in \cC|c} |x|$. Now, we define
$$
\text{mis-clustering ratio at time $t$} := \frac{1}{SA}\sum_{c\in \cC_t} |c\setminus \ell(c)| \, .
$$
Note that the mis-clustering ratio falls in $[0,1]$ as $\sum_{c\in \cC_t} |c|=SA$ for all $t$.
Our second performance measure accounts for the error in the aggregated empirical transition probability due to mis-clustered pairs. We refer to this measure as \emph{mis-clustering bias}. Precisely speaking, for a given pair $e\in \cS\times \cA$, we denote by $c_{e}\in \cC_t$ the set containing $e$ in $\cC_t$. Then, we define the mis-clustering bias at time $t$ as
\als{
\text{mis-clustering bias at time $t$} := \sum_{c\in \cC_t}\sum_{e\notin \ell(c)} \|\hat{p}_t(\cdot|c_e) - \hat{p}_t (\cdot|c_e\setminus \{e\})\|_1 \, .
}
In Figure \ref{fig:clustering_error_CI}, we plot the ``mis-clustering ratio'' and ``mis-clustering bias'' for the empirical equivalence structures produced in the previous experiments. We observes on the figures that the errors in terms of the aforementioned performance measures reduce. These errors do now vanish quickly, thus indicating that the generated empirical equivalence structures do not agree with the true one. Yet, they help reduce uncertainty in the transition probabilities, and, in turn, reduce the regret; we refer to Remark \ref{rem:CUCRL_rem} for a related discussion.

\begin{figure}[t]
\floatconts
  {fig:clustering_error_CI}
  {\caption{Assessment of quality of approximate equivalence structures for Ergodic RiverSwim with 25 and 50 states}}
  {
    \subfigure[Mis-clustering bias]{\label{fig:circle}%
      \includegraphics[width=0.49\linewidth]{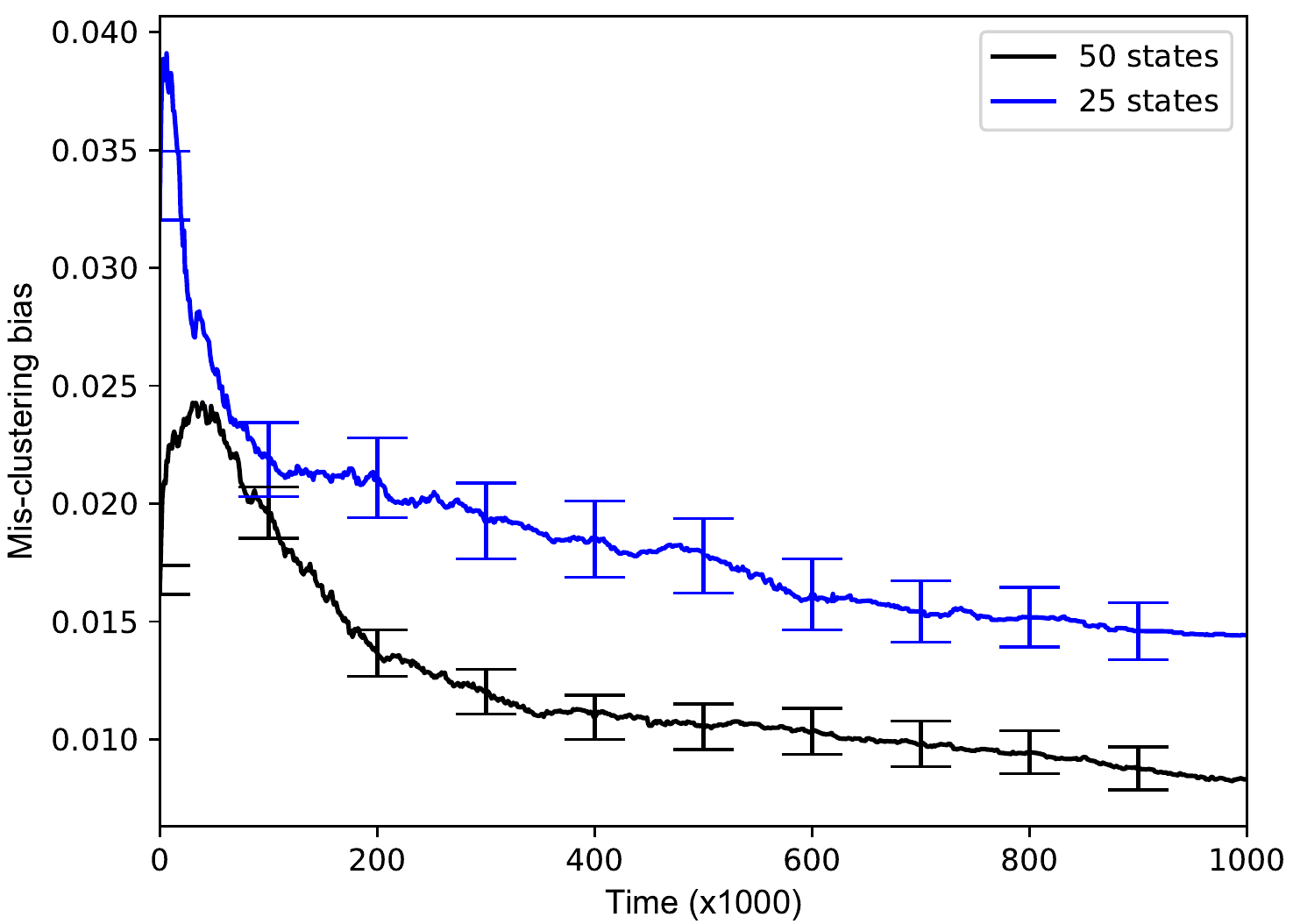}}%
    \subfigure[Mis-clustering ratio]{\label{fig:square}%
      \includegraphics[width=0.49\linewidth]{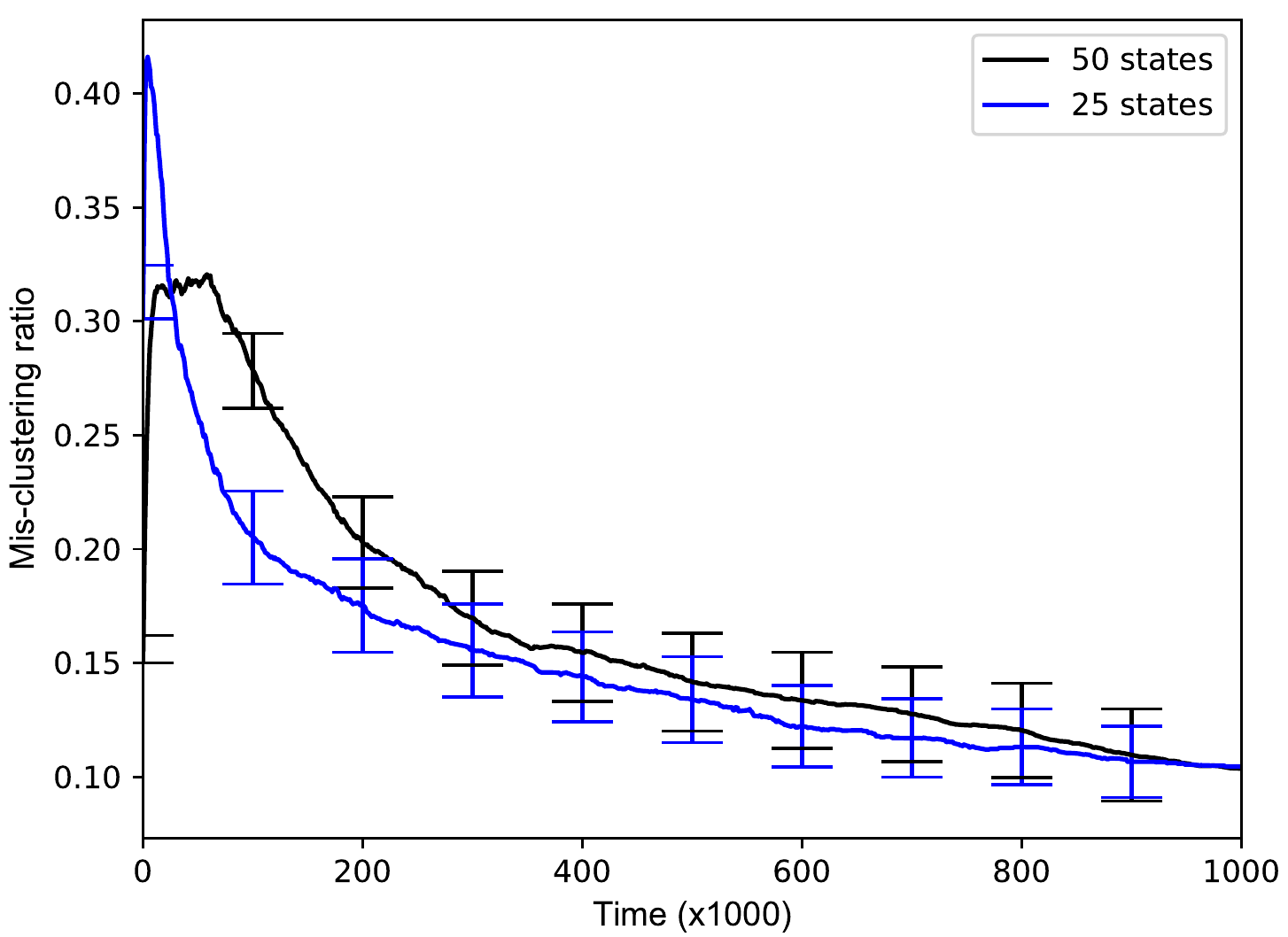}}
  }
\end{figure}


In the second set of experiments, we consider two communicating environments: \emph{4-room grid-world} (with 49 states) and \emph{RiverSwim} (with 25 states). These environments are described in Appendix \ref{app:simul}. In Figure \ref{fig:regret_communicating}, we plot the regret against time steps under \CUCRL$(\cC,\boldsymbol\sigma)$, \CUCRL, and \UCRLLaplace, and similarly to the previous case, we set $\delta = 0.05$ and $\alpha=4$. The results are averaged over 100 independent runs, and the 95\% confidence intervals are shown. In both environments, \CUCRL$(\cC,\boldsymbol\sigma)$ significantly outperforms \UCRLLaplace. However, \CUCRL\ attains a regret, which is slightly worse than that of \UCRLLaplace. This can be attributed to the fact that \ClusteringAlgo\ is unable to find an accurate enough equivalence structure in these non-ergodic environments.

\begin{figure}[htbp]
\floatconts
  {fig:regret_communicating}
  {\caption{Regret of various algorithms in communicating environments}}
  {
    \subfigure[25-state RiverSwim]{\label{fig:circle}%
      \includegraphics[width=0.49\linewidth]{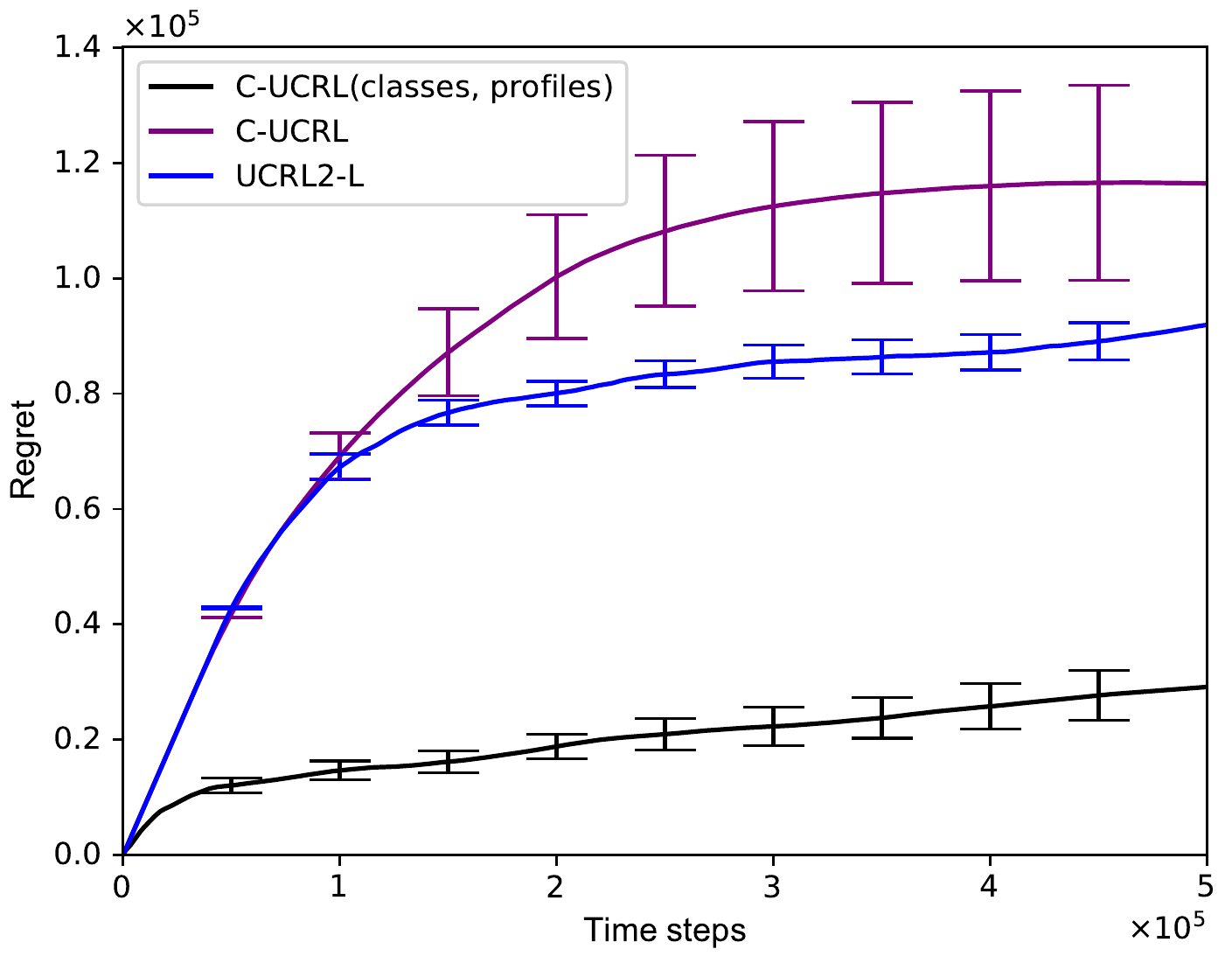}}%
    \subfigure[Grid-World]{\label{fig:square}%
      \includegraphics[width=0.49\linewidth]{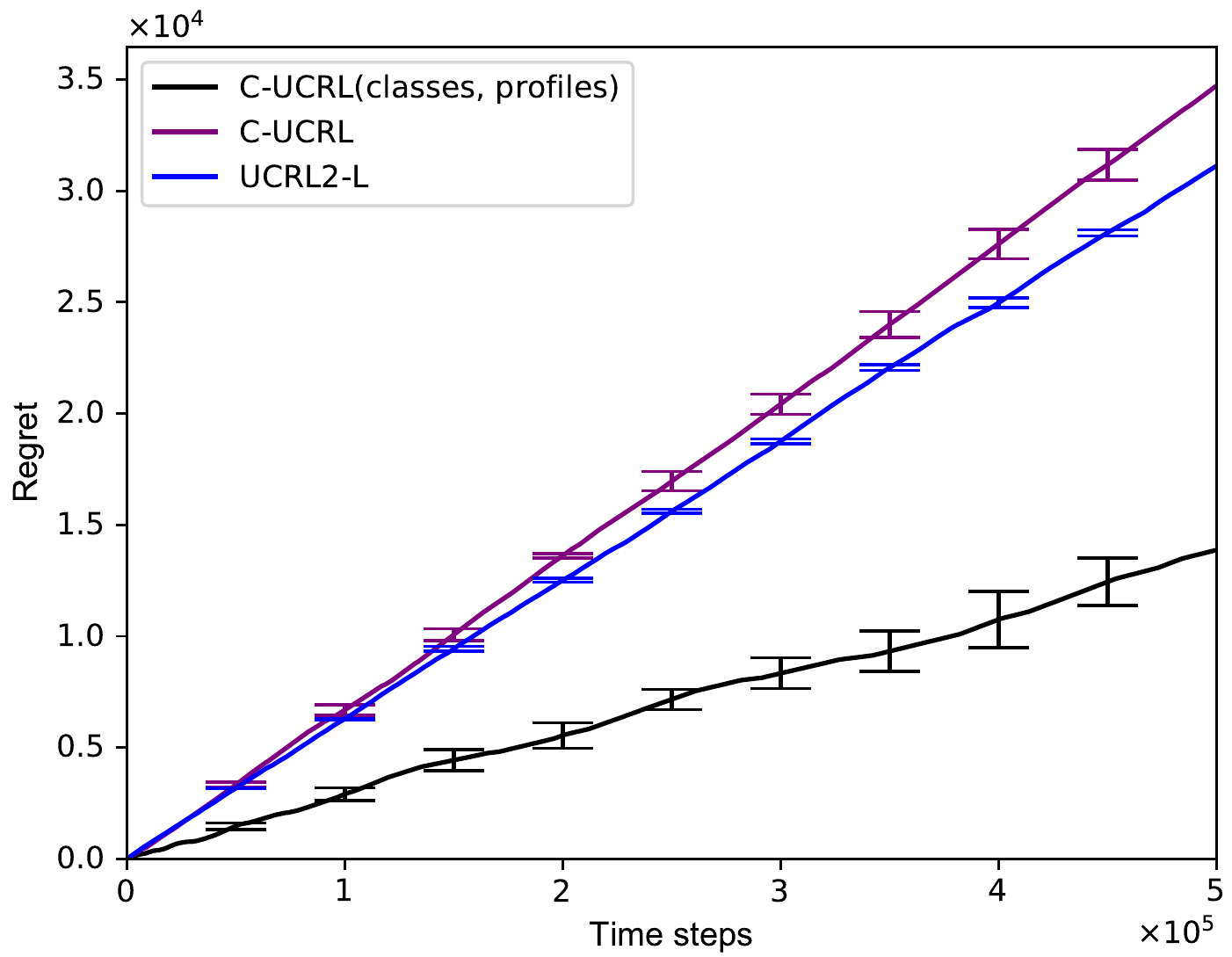}}
  }
\end{figure}

\section{Conclusion}
We introduced a similarity measure of state-action pairs, which induces a notion of equivalence of profile distributions in the state-action space of a Markov Decision Process. 
In the case of a known equivalence structure, we have presented confidence sets incorporating such knowledge that are provably tighter than their corresponding counterparts ignoring equivalence structure. In the case of an unknown equivalence structure, we presented an algorithm, based on confidence bounds, that seeks to estimate an empirical equivalence structure for the MDP. In order to illustrate the efficacy of our developments, we further preseted \CUCRL, which is a natural modification of \UCRL\ using the presented confidence sets. We show that when the equivalence structure is known to the learner, \CUCRL\ attains a regret smaller than that of \UCRL\ by a factor of $\sqrt{SA/C}$ in communicating MDPs, where $C$ denotes the number of classes. 
In the case of an unknown equivalence structure, we show through numerical experiments that in ergodic environments, \CUCRL\ outperforms \UCRL\ significantly. The regret analysis in this case is much more complicated, and we leave it for future work. We believe that the presented confidence sets can be combined with model-based algorithms for the discounted setup, which we expect to yield improved  performance in terms of sample complexity both in theory and practice.

\vspace{-2mm}
\section*{Acknowledgements}
This work has been supported by CPER Nord-Pas-de-Calais/FEDER DATA Advanced data science and technologies 2015-2020, the French Ministry of Higher Education and Research, Inria, and the French Agence Nationale de la Recherche (ANR), under grant ANR-16-CE40-0002 (project BADASS).


\appendix

\section{Pseudo-Codes of \UCRL\ and \CUCRL}\label{app:CUCRL2}
In this section, we provide the pseudo-codes of \UCRL, \CUCRL$(\cC,\boldsymbol\sigma)$, and \CUCRL. 

\begin{algorithm}[!h]
   \caption{\UCRL\ with input parameter $\delta\in (0,1]$ \citep{jaksch2010near} }
   \label{alg:ucrl}
   \footnotesize
\begin{algorithmic}
   \STATE \textbf{Initialize:} For all $(s,a)$, set $N_0(s,a)=0$ and $v_0(s,a)=0$. Set $t_0=0$, $t=1$, $k=1$, and observe the initial state $s_1$;
   \FOR{episodes $k\geq  1$}
       \STATE Set $t_k = t$;
       \STATE Set $N_{t_k}(s,a) = N_{t_{k-1}}(s,a)+ v_{k}(s,a)$ for all $(s,a)$;
       \STATE Compute empirical estimates $\widehat \mu_{t_k}(s,a)$ and $\widehat p_{t_k}(\cdot|s,a)$ for all $(s,a)$;
       \STATE Compute  $\pi^+_{t_k} = \texttt{EVI}\Big(\widehat\mu_{t_k}, \widehat p_{t_k}, N_{t_k}, \tfrac{1}{\sqrt{t_k}},\tfrac{\delta}{SA}\Big)$ --- see Algorithm \ref{alg:EVI};
       \WHILE{$v_{k}(s_t,\pi_{t_k}^+(s_t))<  \max\{1,N_{t_k}(s_t,\pi_{t_k}^+(s_t))\}$}
            \STATE Play action $a_t=\pi_{t_k}^+(s_t)$, and observe the next state $s_{t+1}$ and reward $r_t(s_t, a_t)$;
            \STATE Set $v_k(s_t,a_t)=v_k(s_t,a_t)+1$;
            \STATE Set $t=t+1$;
       \ENDWHILE
   \ENDFOR
\end{algorithmic}
\normalsize
\end{algorithm}


\begin{algorithm}[!h]
	\caption{\texttt{EVI}$(\mu,p,N,\epsilon,\delta)$ \citep{jaksch2010near}}
\footnotesize
	\label{alg:EVI}
	\begin{algorithmic}
		\STATE \textbf{Initialize:} $u^{(0)}\equiv 0, u^{(-1)}\equiv-\infty$, $n=0$;
		\WHILE{$\max_s(u^{(n)}(s)-u^{(n-1)}(s)) - \min_s(u^{(n)}(s)-u^{(n-1)}(s)) > \epsilon$}
		\STATE For all $(s,a)$, set $\mu'(s,a) = \mu(s,a) + \beta'_{N(s,a)}(\delta)$;
        \STATE For all $(s,a)$, set $p'(\cdot|s,a) \in \argmax_{q\in \cP(s,a)} \sum_{x\in \cS} q(x)u^{(n)}(x)$ where
        $$
        \cP(s,a) := \Big\{q\in \Delta^S: \|q - p(\cdot|s,a)\|_1 \leq \beta_{N(s,a)}(\delta)\Big\}\,;
        $$
        \STATE For all $s$, update $u^{(n+1)}(s) =  \max_{a\in\cA} \Big(\mu'(s,a) + \sum_{x\in \cS} p'(x|s,a)u^{(n)}(x)\Big)$;
        \STATE For all $s$, update $\pi_{n+1}(s) \in \Argmax_{a\in\cA} \Big(\mu'(s,a) + \sum_{x\in \cS} p'(x|s,a)u^{(n)}(x)\Big)$;
		\STATE Set $n=n+1$;
\ENDWHILE
\STATE \textbf{Output:} $\pi_{n+1}$
	\end{algorithmic}
\normalsize
\end{algorithm}

\begin{algorithm}[!h]
   \caption{\CUCRL$(\cC,\boldsymbol\sigma)$ with input parameter $\delta\in (0,1]$ }
   \label{alg:cucrl}
   \footnotesize
\begin{algorithmic}
   \STATE \textbf{Initialize:} For all $c\in \cC$, set $n_0(c)=0$ and $V_0(c)=0$. Set $t_0=0$, $t=1$, $k=1$, and observe the initial state $s_1$;
   \FOR{episodes $k\geq  1$}
       \STATE Set $t_k = t$;
       \STATE Set $n_{t_k}(c) = n_{t_k-1}(c) + V_{k-1}(c)$ for all $c$;
       \STATE Compute empirical estimates $\widehat \mu_{t_k}^{\boldsymbol\sigma}(c)$ and $\widehat p_{t_k}^{\boldsymbol\sigma}(\cdot|c)$ for all $c$;
       \STATE Compute  $\pi^+_{t_k} = \texttt{EVI}\Big(\widehat\mu^{\boldsymbol\sigma}_{t_k}, \widehat p^{\boldsymbol\sigma}_{t_k}, n_{t_k}, \tfrac{1}{\sqrt{t_k}},\tfrac{\delta}{C}\Big)$ --- see Algorithm \ref{alg:EVI};
       \WHILE{$V_k(c_t)<\max\{1,n_{t_k}(c_t)\}$}
            \STATE Play action $a_t=\pi_{t_k}^+(s_t)$, and observe the next state $s_{t+1}$ and reward $r_t(s_t, a_t)$;
            \STATE Set $c_t\in \cC$ to be the class containing $(s_t,a_t)$;
            \STATE Set $V_k(c_t)=V_k(c_t)+1$;
            \STATE Set $t=t+1$;
       \ENDWHILE
   \ENDFOR
\end{algorithmic}
\normalsize
\end{algorithm}

\begin{algorithm}[!h]
   \caption{\CUCRL\ with input parameter $\delta\in (0,1]$ }
   \label{alg:cucrl_unknown}
   \footnotesize
\begin{algorithmic}
   \STATE \textbf{Initialize:} For all $(s,a)$, set $N_0(s,a)=0$ and $v_0(s,a)=0$. For all $c\in \cC$, set $n_0(c)=0$ and $V_0(c)=0$. Set $t_0=0$, $t=1$, $k=1$, and observe the initial state $s_1$;
   \FOR{episodes $k\geq  1$}
       \STATE Set $t_k = t$;
       \STATE Set $N_{t_k}(s,a) = N_{t_{k-1}}(s,a)+ v_{k}(s,a)$ for all $(s,a)$;
       \STATE Set $n_{t_k}(c) = n_{t_k-1}(c) + V_{k-1}(c)$ for all $c$;
       \STATE Compute empirical estimates $\boldsymbol\sigma_{t_k}$;
       \STATE Find $\cC_{t_k}$ using \ClusteringAlgo;
       \STATE Compute empirical estimates $\widehat \mu_{t_k}^{\boldsymbol\sigma_{t_k}}(c)$ and $\widehat p_{t_k}^{\boldsymbol\sigma_{t_k}}(\cdot|c)$ for all $c\in \cC_{t_k}$;
       \STATE Compute  $\pi^+_{t_k} = \texttt{EVI}\Big(\widehat\mu^{\boldsymbol\sigma_{t_k}}_{t_k}, \widehat p^{\boldsymbol\sigma_{t_k}}_{t_k}, n_{t_k}, \tfrac{1}{\sqrt{t_k}},\tfrac{\delta}{SA}\Big)$ --- see Algorithm \ref{alg:EVI};
       \WHILE{$v_k(s_t,\pi_{t_k}^+(s_t)) < \max\{1,N_{t_k}(s_t,\pi_{t_k}^+(s_t))\}$ and $V_k(c_t)<\max\{1,n_{t_k}(c_t)\}$}
            \STATE Play action $a_t=\pi_{t_k}^+(s_t)$, and observe the next state $s_{t+1}$ and reward $r_t(s_t, a_t)$;
            \STATE Set $c_t\in \cC_{t_k}$ to be the class containing $(s_t,a_t)$;
            \STATE Set $V_k(c_t)=V_k(c_t)+1$;
            \STATE Set $v_k(s_t,a_t)=v_k(s_t,a_t)+1$;
            \STATE Set $t=t+1$;
       \ENDWHILE
   \ENDFOR
\end{algorithmic}
\normalsize
\end{algorithm}


\section{Proof of Lemma \ref{lem:ordered}}\label{app:ordered_proof}

Let us consider the case when a switch occurs between index $1$ and $2$, that is $ \sigma_q(1)=\sigma_p(2)$ and $\sigma_q(2)=\sigma_p(1)$.	In this situation, we thus have $p(\sigma_p(1))>p(\sigma_p(2))$ but
$p(\sigma_q(1))\leq p(\sigma_q(2))$. Then, we study $\sum_{i=1,2}|p(\sigma_p(i))-q(\sigma_q(i))|$. 	First, we note that if $q(\sigma_q(1))<p(\sigma_p(1))$ and $q(\sigma_q(2))<p(\sigma_p(2))$,
	then
\als{
|p(\sigma_p(1))-q(\sigma_q(1))|+|p(\sigma_p(2))-q(\sigma_q(2)))| &=  p(\sigma_p(1))-q(\sigma_q(2))+ p(\sigma_p(2))-q(\sigma_q(1))\\
&= |p(\sigma_p(1))-q(\sigma_p(1))| +|p(\sigma_p(2))-q(\sigma_p(2))|\,.
}

Likewise, the same equality occurs if $q(\sigma_q(1))>p(\sigma_p(1))$ and $q(\sigma_q(2))>p(\sigma_p(2))$. 	Now, in  the remaining intermediate cases (that is $q(\sigma_p(1))<p(\sigma_p(2))<q(\sigma_p(2))<p(\sigma_p(1))$, 	$p(\sigma_p(2))<q(\sigma_p(1))<q(\sigma_p(2))<p(\sigma_p(1))$, and $p(\sigma_p(2))<q(\sigma_p(1))<p(\sigma_p(1))<q(\sigma_p(2))$),
it is immediate to check that
\beqan
|p(\sigma_p(1))-q(\sigma_q(1))|+|p(\sigma_p(2))-q(\sigma_q(2)))| \leq |p(\sigma_p(1))-q(\sigma_p(1))| + |p(\sigma_p(2))-q(\sigma_p(2))|\,.
\eeqan
	
Thus, proceeding iteratively for all switch that occurs, and decomposing the permutations $\sigma_p$ and $\sigma_q$ into elementary switches, we deduce that  almost surely
\als{
\| q(\sigma_q(\cdot)) - p(\sigma_p(\cdot))\|_1 \leq \| q(\sigma_p(\cdot)) - p(\sigma_p(\cdot))\|_1 = \|p - q\|_1\,,
}
thus concluding the lemma.
\ep


\section{Proof of Proposition \ref{prop:clustering_guarantee}}\label{app:clustering}
First recall that
$
\Delta:= \min\big\{d(\{\ell\},\{\ell'\}): \ell,\ell'\in \cS\times \cA \hbox{ and $\ell, \ell'$ are not in the same class}\big\}\,.
$
Define
\als{
\cE=\bigcap_{t\in \Nat}\bigcap_{s,a}\Big\{\|p(\cdot|s,a) - \widehat p_t(\cdot|s,a)\|_1 \leq \beta_{N_t(s,a)}\big(\tfrac{\delta}{SA}\big) \Big\}\, .
}
Note that $\mathbb P(\cE)\geq 1-\delta$. We will need the following lemma:

\begin{lemma}
\label{lem:clustering_concentration}
Assume that the event $\cE$ holds. Then, for all $t$, at every round $k$ of \emph{\ClusteringAlgo}, for all $v\in \cC^k$, there exists $u\in\cN(v)$ such that $u$ and $v$ belong to the same class.
\end{lemma}

\begin{proof}\textbf{(of Lemma \ref{lem:clustering_concentration})}
Fix $t\geq 1$ and round $k$, and consider $v\in \cC^k$. Recall that $u$ is a PAC Neighbor of $v$ if it satisfies:
\begin{align*}
(i)&\quad \big\|\widehat p^{\boldsymbol\sigma_{u,t}}_t(\cdot|u) - \widehat p^{\boldsymbol\sigma_{v,t}}_t(\cdot|v)\big\|_1 - \epsilon_{u,t} - \epsilon_{v,t} \leq 0\, ; \\
(ii)&\quad \big\|\widehat p_t(\sigma_{i,t}(\cdot)|i) - \widehat p_t(\sigma_{j,t}(\cdot)|j)\big\|_1 - \beta_{N_t(i)}\big(\tfrac{\delta}{SA}\big) - \beta_{N_t(j)}\big(\tfrac{\delta}{SA}\big) \leq 0\, ,\quad \forall i\in u,\forall j\in v; \\
(iii)&\quad \big\|\widehat p_t(\sigma_{\ell,t}(\cdot)|\ell) - \widehat p^{\boldsymbol\sigma_{u\cup v,t}}_t(\cdot|u \cup v)\big\|_1 - \beta_{N_t(\ell)}\big(\tfrac{\delta}{SA}\big) - \epsilon_{u\cup v,t} \leq 0\, ,\quad \forall \ell\in u\cup v.
\end{align*}

In order to prove the lemma, it suffices to show that under $\cE$, there exists $u\subset \cS\times \cA$ satisfying  $(i)$--$(iii)$ and $d(u,v)=0$. To this end, we will show that the event $\cE$ implies the following:
For all $u\in \cS\times \cA$,
\begin{align*}
(i')&\quad \big\|\widehat p^{\boldsymbol\sigma_{u,t}}_t(\cdot|u) - \widehat p^{\boldsymbol\sigma_{v,t}}_t(\cdot|v)\big\|_1 \leq d(u,v) + \epsilon_{u,t} + \epsilon_{v,t}\, ; \\
(ii')&\quad \big\|\widehat p_t(\sigma_{i,t}(\cdot)|i) - \widehat p_t(\sigma_{j,t}(\cdot)|j)\big\|_1 \leq d(\{i\},\{j\}) + \beta_{N_t(i)}\big(\tfrac{\delta}{SA}\big) + \beta_{N_t(j)}\big(\tfrac{\delta}{SA}\big) \, ,\quad \forall i\in u,\forall j\in v; \\
(iii')&\quad \big\|\widehat p_t(\sigma_{\ell,t}(\cdot)|\ell) - \widehat p^{\boldsymbol\sigma_{u\cup v,t}}_t(\cdot|u \cup v)\big\|_1 \leq d(\{\ell\},u\cup v) +  \beta_{N_t(\ell)}\big(\tfrac{\delta}{SA}\big) + \epsilon_{u\cup v,t}\, ,\quad \forall \ell\in u\cup v.
\end{align*}

Now, $(i')$--$(iii')$ imply that there exists $u\in \cN(v)$ such that $u$ and $v$ belong to the same class, and the lemma follows. It remains to prove $(i')$--$(iii')$.

\paragraph{Proof of $(i')$.}Consider $u\in \cS\times \cA$. Then, the non-expansive property of the norm function implies
\als{
\big\|\widehat p^{\boldsymbol\sigma_{u,t}}_t(\cdot|u) - \widehat p^{\boldsymbol\sigma_{v,t}}_t(\cdot|v)\big\|_1 &\leq \big\|p^{\boldsymbol\sigma_{u}}(\cdot|u) - p^{\boldsymbol\sigma_{v}}(\cdot|v)\big\|_1+ \big\|\widehat p^{\boldsymbol\sigma_{u,t}}_t(\cdot|u) - p^{\boldsymbol\sigma_{u}}(\cdot|u)\big\|_1 + \big\|\widehat p^{\boldsymbol\sigma_{v,t}}_t(\cdot|v) - p^{\boldsymbol\sigma_{v}}(\cdot|v)\big\|_1 \\
&= d(u,v) + \underbrace{\big\|\widehat p^{\boldsymbol\sigma_{u,t}}_t(\cdot|u) - p^{\boldsymbol\sigma_{u}}(\cdot|u)\big\|_1}_{A_1} + \underbrace{\big\|\widehat p^{\boldsymbol\sigma_{v,t}}_t(\cdot|v) - p^{\boldsymbol\sigma_{v}}(\cdot|v)\big\|_1}_{A_2} \, .
}
The term $A_1$ is upper bounded as follows:
\als{
A_1 &= \sum_{x\in \cS} \big\|\widehat p^{\boldsymbol\sigma_{u,t}}_t(x|u) - p^{\boldsymbol\sigma_{u}}(x|u)\big\|_1 \\
&= \sum_{x\in \cS} \bigg|\frac{1}{n_t(u)}\sum_{(s,a)\in u} N_t(s,a)\Big(\widehat p_t(\sigma_{s,a,t}(x)|s,a) - p(\sigma_{s,a}(x)|s,a)\Big) \bigg| \\
&\leq \frac{1}{n_t(u)} \sum_{(s,a)\in u} N_t(s,a)\sum_{x\in \cS} \Big|\widehat p_t(\sigma_{s,a,t}(x)|s,a) - p(\sigma_{s,a}(x)|s,a)\Big|  \\
&\leq \frac{1}{n_t(u)} \sum_{(s,a)\in u} N_t(s,a)\big\|\widehat p_t(\sigma_{s,a,t}(\cdot)|s,a) - p(\sigma_{s,a}(\cdot)|s,a)\big\|_1 \\
&\leq \frac{1}{n_t(u)} \sum_{(s,a)\in u}N_t(s,a) \big\|\widehat p_t(\cdot|s,a) - p(\cdot|s,a)\big\|_1 \, ,
}
where we have used Lemma \ref{lem:ordered} as well as the non-expansive property of the norm function.
Hence, under the event $\cE$,
\begin{equation}\label{eq:A_1}
\big\|\widehat p^{\boldsymbol\sigma_{u,t}}_t(\cdot|u) - p^{\boldsymbol\sigma_{u}}(\cdot|u)\big\|_1
\leq \frac{1}{n_t(u)} \sum_{(s,a)\in u}N_t(s,a)\beta_{N_t(s,a)}\big(\tfrac{\delta}{SA}\big) = \epsilon_{u,t} \, .
\end{equation}

A similar argument yields $A_2\leq \epsilon_{v,t}$ under $\cE$. Putting these together verifies $(i')$.

\paragraph{Proof of $(ii')$.} The proof of $(ii')$ is quite similar to that of $(i')$, hence omitted.

\paragraph{Proof of $(iii')$.} Consider $u\in \cS\times \cA$ and $\ell\in u\cup v$. We have
\als{
\big\|\widehat p_t(\sigma_{\ell,t}(\cdot)|\ell) - \widehat p^{\boldsymbol\sigma_{u\cup v,t}}_t&(\cdot|u \cup v)\big\|_1 \leq
\big\|p(\sigma_{\ell}(\cdot)|\ell) -  p^{\boldsymbol\sigma_{u\cup v}}(\cdot|u \cup v)\big\|_1 +
\big\|\widehat p_t(\sigma_{\ell,t}(\cdot)|\ell) - p(\sigma_{\ell}(\cdot)|\ell)\big\| \\
& + \big\| \widehat p^{\boldsymbol\sigma_{u\cup v}}(\cdot|u \cup v) - \widehat p^{\boldsymbol\sigma_{u\cup v,t}}_t(\cdot|u \cup v)\big\|_1 \\
&\leq d(\{\ell\},u\cup v) + \|\widehat p_t(\cdot|\ell) - p(\cdot|\ell)\|_1
+ \big\|  p^{\boldsymbol\sigma_{u\cup v}}(\cdot|u \cup v) - \widehat p^{\boldsymbol\sigma_{u\cup v,t}}_t(\cdot|u \cup v)\big\|_1\, ,
}
where we have used Lemma \ref{lem:ordered} and the non-expansive property of the norm function.
The third term in the right-hand side is bounded as follows:
\als{
\big\|  p^{\boldsymbol\sigma_{u\cup v}}(\cdot|u \cup v) - \widehat p^{\boldsymbol\sigma_{u\cup v,t}}_t(\cdot|u \cup v)\big\|_1
&\leq \sum_{(s,a)\in u\cup v} \frac{N_t(s,a)}{n_t(u) + n_t(v)}\sum_{x\in \cS} \Big|p(\sigma_{s,a}(x)|x) - \widehat p_t(\sigma_{s,a,t}(x)|s,a)\Big| \\
&= \sum_{(s,a)\in u\cup v} \frac{N_t(s,a)}{n_t(u) + n_t(v)} \big\|\widehat p_t(\sigma_{s,a,t}(\cdot)|s,a) - p(\sigma_{s,a}(\cdot)|s,a)\big\|_1 \\
&\leq \sum_{(s,a)\in u\cup v} \frac{N_t(s,a)}{n_t(u) + n_t(v)} \big\|\widehat p_t(\cdot|s,a) - p(\cdot|s,a)\big\|_1 \, .
}
Hence, when $\cE$ occurs,
$\big\|  p^{\boldsymbol\sigma_{u\cup v}}(\cdot|u \cup v) - \widehat p^{\boldsymbol\sigma_{u\cup v,t}}_t(\cdot|u \cup v)\big\|_1
\leq \epsilon_{u\cup v,t}$, so that
\als{
\big\|\widehat p_t(\sigma_{\ell,t}(\cdot)|\ell) - \widehat p^{\boldsymbol\sigma_{u\cup v,t}}_t(\cdot|u \cup v)\big\|_1 &\leq d(\{\ell\},u\cup v) + \beta_{N_t(\ell)}\big(\tfrac{\delta}{SA}\big) + \epsilon_{u\cup v,t}\, .
}
\end{proof}

We are now ready to prove the proposition.

\paragraph{Proof (of Proposition \ref{prop:clustering_guarantee})}
Fix $t\geq 1$, and consider $\alpha \to\infty$ (the choice $\alpha \geq \frac{t}{\max\{1,f^{-1}(\Delta)\}}$ suffices).  Assume that $\min_{s,a} N_t(s,a)> f^{-1}(\Delta)$, and that $\cE$ holds.
By Lemma \ref{lem:clustering_concentration}, we have that at any round of the algorithm, the set of PAC Neighbors of a given $v\in \cS\times\cA$ maintained by the algorithm contains some $u\in \cS\times \cA$ belonging to the same class as $v$.

We prove the theorem by induction. First we show that the best case holds, that is in the first iteration of the algorithm, (i) the algorithm avoids grouping state-action pairs belonging to different classes; and (ii) the algorithm groups all the pairs in the same class. Initially, all the classes are singletons. So in the first iteration, the algorithm starts with the  classes sorted according to a non-increasing order of number of samples, and then iteratively merges each class with its PAC Nearest Neighbor (see Definition \ref{def:PAC_NN}). Recall that for a partition $\cC$, ${\textsf{Near}}(c,\cC)$ denotes the PAC Nearest Neighbor of $\cC$:
$	{\textsf{Near}}(c,\cC) \in \argmin_{x\in \cN(c)} \widehat d(c,x)$. In the first round the algorithm, if $i,j\in \cS\times \cA$ are combined, then $\widehat d(\{i\},\{j\})\leq 0$. In view of the definition of $\widehat d(\cdot,\cdot)$, we deduce that
\begin{align}
\big\|\widehat p_t(\sigma_{i,t}(\cdot)|i) - \widehat p_t(\sigma_{j,t}(\cdot)|j)\big\|_1 - \beta_{N_t(i)}\big(\tfrac{\delta}{SA}\big) - \beta_{N_t(j)}\big(\tfrac{\delta}{SA}\big) \leq 0\, .
\label{eq:base_case}
\end{align}

In order to show that the algorithm makes no mistake, we need to show that $d(\{i\},\{j\})=0$. We have
\als{
d(\{i\}&,\{j\})= \|p(\sigma_i(\cdot)|i) - p(\sigma_{j}(\cdot)|j)\|_1 \\
&\leq  \|p(\sigma_i(\cdot)|i) - \widehat p_t(\sigma_{i}(\cdot)|i)\|_1 + \|p(\sigma_j(\cdot)|j) - \widehat p_t(\sigma_{j,t}(\cdot)|j)\|_1 + \|\widehat p_t(\sigma_{i,t}(\cdot)|i) - \widehat p_t(\sigma_{j,t}(\cdot)|j)\|_1\\
&\leq  \|p(\sigma_i(\cdot)|i) - \widehat p_t(\sigma_{i}(\cdot)|i)\|_1 + \|p(\sigma_j(\cdot)|j) - \widehat p_t(\sigma_{j}(\cdot)|j)\|_1 + \|\widehat p_t(\sigma_{i,t}(\cdot)|i) - \widehat p_t(\sigma_{j,t}(\cdot)|j)\|_1\\
&=  \|p(\cdot|i) - \widehat p_t(\cdot|i)\|_1 + \|p(\cdot|j) - \widehat p_t(\cdot|j)\|_1 + \|\widehat p_t(\sigma_{i,t}(\cdot)|i) - \widehat p_t(\sigma_{j,t}(\cdot)|j)\|_1\, ,
}
where the first inequality follows from the sub-additivity of the norm function, and the second follows from Lemma \ref{lem:ordered}.
Hence, under the event $\cE$, it holds that
\als{
d(\{i\},\{j\}) \leq   \beta_{N_t(i)}\big(\tfrac{\delta}{SA}\big) +  \beta_{N_t(j)}\big(\tfrac{\delta}{SA}\big) +  \|\widehat p_t(\sigma_{i,t}(\cdot)|i) - \widehat p_t(\sigma_{j,t}(\cdot)|j)\|_1\, .
}
Combining this with (\ref{eq:base_case}), we have under $\cE$,
\als{
d(\{i\},\{j\}) \leq 2 \beta_{N_t(i)}\big(\tfrac{\delta}{SA}\big) + 2 \beta_{N_t(j)}\big(\tfrac{\delta}{SA}\big) \, .
}
In view of the assumption $\min_{s,a} N_t(s,a)> f^{-1}(\Delta)$, and noting that $d(\{i\},\{j\})\geq \Delta$, we deduce that $d(\{i\},\{j\})\leq 0$, so that the base case holds.

Now assume that at the end of iteration $m$, the algorithm outputs a valid partition under $\cE$, namely,  it does not wrongly group pairs coming from different classes. We would like to show that  the partition obtained in iteration $m+1$ is valid, too.
To this end, consider $u,v\in \cC^{m}$ that are merged by the algorithm in round $m+1$, so that $u\cup v \in \cC^{m+1}$. First note that by Lemma \ref{lem:clustering_concentration}, for any $v\in \cC^m$, there exists $u'\in \cN(v)$ with $d(u',v)=0$. We need to show that $d(u,v)=0$. By construction, $u=\textsf{Near}(v,\cC^m)$, and so the following inequalities hold:
\begin{align*}
&\big\|\widehat p^{\boldsymbol\sigma_{u,t}}_t(\cdot|u) - \widehat p^{\boldsymbol\sigma_{v,t}}_t(\cdot|v)\big\|_1 - \epsilon_{u,t} - \epsilon_{v,t} \leq 0\, ; \\
&\big\|\widehat p_t(\sigma_{i,t}(\cdot)|i) - \widehat p_t(\sigma_{j,t}(\cdot)|j)\big\|_1 - \beta_{N_t(i)}\big(\tfrac{\delta}{SA}\big) - \beta_{N_t(j)}\big(\tfrac{\delta}{SA}\big)\leq 0\, ,\quad \forall i\in u,\forall j\in v; \\
&\big\|\widehat p_t(\sigma_{\ell,t}(\cdot)|\ell) - \widehat p^{\boldsymbol\sigma_{u\cup v,t}}_t(\cdot|u \cup v)\big\|_1 - \beta_{N_t(\ell)}\big(\tfrac{\delta}{SA}\big) - \epsilon_{u\cup v,t} \leq 0\, ,\quad \forall \ell\in u\cup v.
\end{align*}
Using similar steps as in the proof of Lemma \ref{lem:clustering_concentration}, it follows that
\als{
d(u,v) &\leq
  \|p^{\boldsymbol\sigma_u}(\cdot|u) - \widehat p^{\boldsymbol\sigma_{u,t}}_t(\cdot|u)\|_1 + \|p^{\boldsymbol\sigma_{v}}(\cdot|v) - \widehat p^{\boldsymbol\sigma_{v,t}}_t(\cdot|v)\|_1 + \|\widehat p^{\boldsymbol\sigma_{u,t}}_t(\cdot)|u) - \widehat p^{\boldsymbol\sigma_{v,t}}_t(\cdot|v)\|_1\, .
}
Using (\ref{eq:A_1}) in the proof of Lemma \ref{lem:clustering_concentration}, we arrive at
\als{
d(u,v) &\leq  \epsilon_{u,t} + \epsilon_{v,t} +  \big\|\widehat p^{\boldsymbol\sigma_{u,t}}_t(\cdot|u) - \widehat p^{\boldsymbol\sigma_{v,t}}_t(\cdot|v)\big\|_1 \\
&\leq 2\epsilon_{u,t} + 2\epsilon_{v,t} < 4\beta_{f^{-1}(\Delta)}\big(\tfrac{\delta}{SA}\big) \, .
}
We thus deduce that $d(u,v)= 0$, which  concludes the proof.
\ep


\section{Regret Analysis of \CUCRL$(\cC,\sigma)$: Proof of Theorem \ref{thm:regretKnownC}}\label{app:regretbound}

In this section, we prove Theorem \ref{thm:regretKnownC}, which provides an upper bound on the regret of \CUCRL$(\cC,\boldsymbol\sigma)$. We provide the proof for the case when the reward function is unknown to the learner too.
Our proof follows similar lines as in the proof of \cite[Theorem~2]{jaksch2010near}.
We first provide the following time-uniform concentration inequality to control a bounded martingale difference sequence, which follows from time-uniform Laplace concentration inequality:


\begin{lemma}[Time-uniform Azuma-Hoeffding]\label{lem:time-uniform-AzumaHoeffding}
	Let $(X_t)_{t\geq 1}$ be a martingale difference sequence bounded by $b$ for some $b>0$  (that is, $|X_t|\leq b$ for all $t$). Then, for all $\delta\in (0,1)$,
	\beqan
	\mathbb P\bigg(\exists T \in\Nat: \sum_{t=1}^T  X_t \geq b\sqrt{ 2(T+1) \log\big( \sqrt{T+1}/\delta\big)}\bigg) \leq \delta\,.
	\eeqan
\end{lemma}

\begin{proof}\textbf{(of Theorem \ref{thm:regretKnownC})} Let $\delta\in (0,1)$. To simplify notations, we define the short-hand $J_k:=J_{t_k}$ for various random variables that are fixed within a given episode $k$ (for example $\cM_{k}:=\cM_{t_k}$). Denote by $m(T)$ the number of episodes initiated by the algorithm up to time $T$.
An application of Lemma \ref{lem:time-uniform-AzumaHoeffding} yields:
\als{
\kR(T)&= \sum_{t=1}^T g_\star- \sum_{t=1}^T r_t(s_t,a_t) \leq \sum_{s,a} N_{m(T)}(s,a)(g_\star - \mu(s,a)) + \sqrt{ \tfrac{1}{2}(T+1) \log( \sqrt{T+1}/\delta)}\, ,
}
with probability at least $1-\delta$. We have
\als{
\sum_{s,a} N_{m(T)}(s,a)(g_\star - \mu(s,a)) &= \sum_{k=1}^{m(T)}  \sum_{s,a}  \sum_{t=t_k+1}^{t_{k+1}}\indic{s_t=s,a_t=a} \big(g_\star - \mu(s,a)\big)\\
&= \sum_{k=1}^{m(T)}  \sum_{s,a}  \nu_k(s,a)\big(g_\star - \mu(s,a)\big)\,.
}
Defining $\nu_k(c):=\sum_{s,a} \nu_k(s,a)$ for $c\in \cC$, we further obtain
\als{
\sum_{s,a} N_{m(T)}(s,a)(g_\star - \mu(s,a)) &= \sum_{k=1}^{m(T)}  \sum_{c\in\cC}  \nu_k(c) \big(g_\star - \mu(c)\big) \, ,
}
where we have used that $\mu(s,a)$ has constant value $\mu(c)$ for all $(s,a)\in c$. For $1\leq k\leq m(T)$, we define the regret of episode $k$ as
$\Delta_k = \sum_{c\in\cC}\nu_k(c) \big(g_\star - \mu(c)\big)$. Hence, with probability at least $1-\delta$,
\als{
\kR(T) \leq \sum_{k=1}^{m(T)} \Delta_k + \sqrt{ \tfrac{1}{2}(T+1) \log( \sqrt{T+1}/\delta)}\, .
}
We say an episode is \emph{good} if $M \in \cM_{k}$ (that is, the set $\cM_{k}$ of plausible MDPs contains the true model), and \emph{bad} otherwise.

\paragraph{Control of the regret due to bad episodes ($M \notin \cM_{k}$).}
Due to using time-uniform instead of time-instantaneous confidence bounds, we can show that with high probability, all episodes are good for $T\in \Nat$. More precisely,
with probability higher than $1-2\delta$, for all $T$, bad episodes do not contribute to the regret:
\beqan
\sum_{k=1}^{m(T)}\Delta_k\indic{M \notin \cM_{k}} = 0\,.
\eeqan

\paragraph{Control of the regret due to good episodes ($M \in \cM_{k}$).} We closely follow \citep{jaksch2010near} and decompose the regret to control the transition and reward functions.
At a high level, we make two major modifications as follows. (i) We use the time-uniform bound stated in Lemma \ref{lem:time-uniform-AzumaHoeffding} to control the martingale difference sequence that appears; and
(ii) as the stopping criterion of \CUCRL($\cC,\boldsymbol\sigma$) slightly differs from that of \UCRL, we use the following lemma to control the number $m(T)$ of episodes:
\begin{lemma}[Number of episodes]\label{lem:NbEpisodes}
	The number $m(T)$ of episodes of \CUCRL($\cC,\boldsymbol\sigma$) up to time $T \geq C$ is upper bounded by:
	\begin{equation*}
	m(T) \leq C\log_2(\tfrac{8T}{C})\, .
	\end{equation*}
\end{lemma}

Consider a good episode $k$ (hence, $M \in \cM_{k}$). The \texttt{EVI} algorithm outputs a policy $\pi^+_{k}$ and $\widetilde M_{k}$ satisfying  $g_{\pi^+_{k}}^{\widetilde M_{k}} \geq g_\star - \frac{1}{\sqrt{t_k}}$. Let us define $g_k:= g_{\pi^+_{k}}^{\widetilde M_{k}}$. It then follows that
\begin{equation}
\Delta_k = \sum_{c \in \cC} \nu_k(c) \big(g_\star - \mu(c)\big) \leq \sum_{c \in \cC} \nu_k(c)
\big(g_k - \mu(c)\big) + \sum_{c \in \cC} \frac{\nu_k(c)}{\sqrt{t_k}} \, .
\label{eq:delta_init}
\end{equation}
Using the same argument as in the proof of \cite[Theorem~2]{jaksch2010near}, the value function $u_k^{(i)}$ computed by \texttt{EVI} at the last iteration $i$ satisfies: $\max_{s} u_k^{(i)}(s) - \min_{s} u_k^{(i)}(s) \!\leq\! D$. Moreover, the convergence criterion of \texttt{EVI} implies
\begin{equation}
|u_k^{(i+1)}(s) - u_k^{(i)}(s) - g_k| \leq \frac{1}{\sqrt{t_k}}\,, \qquad  \forall s \in \cS\, .
\label{eq:puterman_convergence}
\end{equation}

By the design of \texttt{EVI}, we have
$
u_k^{(i+1)}(s) = \widetilde\mu_{k}(s,\pi^+_{k}(s)) + \sum_{x} \widetilde p_k(x|s,\pi^+_{k}(s)) u_k^{(i)}(x) \,.
$
Substituting this into (\ref{eq:puterman_convergence}) gives
\begin{equation*}
\Big|\Big( g_k - \widetilde\mu_{k}(s,\pi_k^+(s)) \Big) - \Big(\sum_{x} \widetilde p_{k}(x|s,\pi_k^+(s))u_k^{(i)}(x) - u_k^{(i)}(s)\Big)\Big|
\leq \frac{1}{\sqrt{t_k}}\, , \qquad \forall s\in \cS\, .
\end{equation*}
Defining
${\bf g}_k = g_k \mathbf 1$, $\widetilde{\boldsymbol{\mu}}_{k} := \big(\widetilde\mu_{k}(s,\pi_k^+(s))\big)_{s\in \cS}$, $\widetilde{\mathbf{P}}_k := \big(\widetilde p_{k}(x|s,\pi_k^+(s))\big)_{s, x\in \cS}$, and $\nu_{k} := \big(\nu_k\big(s,\pi_k^+(s))_{s\in \cS}$, we can rewrite the above inequality as:
\begin{equation*}
\Big|{\bf g}_k - \widetilde{\boldsymbol{\mu}}_{k} - (\widetilde{\mathbf{P}}_k - \mathbf{I}) u_k^{(i)} \Big|
\leq \frac{1}{\sqrt{t_k}}\mathbf 1\,.
\end{equation*}
Combining this with (\ref{eq:delta_init}) yields
\begin{align}
\Delta_k&\leq \sum_{s,a} \nu_k(s,a) \big(g_k-\mu(s,a)\big) + \sum_{s,a} \frac{\nu_k(s,a)}{\sqrt{t_k}} \nonumber\\
&= \sum_{s,a} \nu_k(s,a) \big( g_k - \widetilde\mu_{k}(s,a) \big) + \sum_{s,a} \nu_k(s,a) \big( \widetilde\mu_{k}(s,a) - \mu(s,a) \big) + \sum_{s,a}
\frac{\nu_k(s,a)}{\sqrt{t_k}} \nonumber\\
&\leq \nu_{k} (\widetilde{\mathbf{P}}_k - \mathbf{I} ) u_k^{(i)} + \sum_{s,a} \nu_k(s,a) \big(\widetilde\mu_{k}(s,a) - \mu(s,a) \big) + 2 \sum_{s,a} \frac{\nu_k(s,a)}{\sqrt{t_k}}\, .\nonumber
\end{align}
Similarly to \citep{jaksch2010near}, we define $w_k(s) := u_k^{(i)}(s) - \tfrac{1}{2}(\min_s u_k^{(i)}(s) + \max_s u_k^{(i)}(s))$ for all $s\in \cS$. Then, in view of the fact that $\widetilde{\mathbf{P}}_k$ is row-stochastic, we obtain

\begin{align}
\Delta_k&\leq \nu_{k} (\widetilde{\mathbf{P}}_k - \mathbf{I} ) w_k + \sum_{s,a} \nu_k(s,a) \big(\widetilde \mu_{k}(s,a) - \mu(s,a) \big) + 2 \sum_{s,a} \frac{\nu_k(s,a)}{\sqrt{t_k}}\, .
\end{align}
The second term in the right-hand side can be upper bounded as follows. Fix pair $(s,a)$ and let $c_{s,a}$ denote the cluster to which $(s,a)$ belongs. The fact $M \in \cM_{k}$ implies
\als{
\widetilde\mu_{k}(s,a) - \mu(s,a) &\leq |\widetilde\mu_{k}(s,a) - \widehat \mu_{k}(s,a)| + |\widehat \mu_{k}(s,a) - \mu(s,a)| \leq 2\beta'_{n_{k}(c_{s,a})}(\tfrac{\delta}{C}) \\
    &= 2\sqrt{\frac{1}{2n_{k}(c_{s,a})}\Big(1\!+\!\frac{1}{n_{k}(c_{s,a})}\Big)\log\Big(C\sqrt{n_{k}(c_{s,a})\!+\!1}/\delta\Big)} \\
    &\leq 2\sqrt{\frac{1}{n_k(c_{s,a})}\log\big(C\sqrt{T+1}/\delta\big)}\, ,
}
where we have used $1\leq n_k(c_{s,a}) \leq T$ in the last inequality. Using this bound and noting that $t_k\geq n_k(c)$, we obtain
\begin{align}
\Delta_k
&\leq \nu_{k} (\widetilde{\mathbf{P}}_k-\mathbf{I})w_k + 2\Big(\sqrt{\log\big(C\sqrt{T+1}/\delta\big)} + 1\Big)
\sum_{c\in \cC} \frac{\nu_k(c)}{\sqrt{n_{k}(c)}}\, .
\label{eq:main_delta_minmk}
\end{align}

In what follows, we derive an upper bound on $\nu_{k} (\widetilde{\mathbf{P}}_k-\mathbf{I})w_k$. Similarly to \citep{jaksch2010near}, we consider the following decomposition:
\als{
\nu_k(\widetilde{\mathbf{P}}_k - \mathbf{I}) w_k = \underbrace{\nu_k (\widetilde{\mathbf{P}}_k - \mathbf{P}_k) w_k}_{L_1(k)} + \underbrace{\nu_k (\mathbf{P}_k-\mathbf{I})w_k}_{L_2(k)} \, .
}
Noting that $\|w_k\|_\infty \leq \frac{D}{2}$, we upper bound $L_1(k)$ as follows:
\begin{align}
L_1(k) &\leq \sum_{s,a}\nu_k(s,a) \big(\widetilde p_{k}(s'|s,a) - p(s'|s,a)\big) w_k(s') \sk
&\leq \sum_{s,a} \nu_k(s,a) \|\widetilde p_{k}(\cdot|s,a)- p(\cdot|s,a) \|_1 \| w_k\|_\infty \sk
&\leq D \sum_{s,a} \nu_k(s,a) \bW_{n_{k}(c_{s,a})}(\tfrac{\delta}{C}) \sk
&= D \sum_{c\in \cC} \nu_k(c) \bW_{n_{k}(c)}(\tfrac{\delta}{C}) \sk
&\leq 2D \sqrt{\log\big(C2^{S}\sqrt{T+1}/\delta\big)} \sum_{c\in \cC} \frac{\nu_k(c)}{\sqrt{n_{k}(c)}} \, .
\label{eq:trans_prob_first_bound_minmk}
\end{align}

To upper bound $L_2(k)$, similarly to the
proof of \citep[Theorem~2]{jaksch2010near}, we define the  sequence $(X_t)_{t\geq 1}$, with $X_t := (p(\cdot|s_t,a_t) - \mathbf e_{s_{t+1}})w_{k_t}\bI\{M \in \cM_{k_t}\}$, for all $t$, where $k_t$ denotes the episode containing step $t$. Note that $\mathbb{E}[X_t|s_1, a_1, \dots, s_t, a_t] = 0$, so $(X_t)_{t\geq 1}$ is martingale difference sequence. Furthermore, $|X_t|\leq D$: Indeed, for all $t$, by the H\"{o}lder inequality,
$$
|X_t| \leq \|p(\cdot|s_t,a_t) - \mathbf e_{s_{t+1}}\|_1\|w_{k(t)}\|_\infty \leq \Big(\|p(\cdot|s_t,a_t)\|_1 + \|\mathbf e_{s_{t+1}}\|_1\Big)\frac{D}{2}  = D\,.
$$
Using similar steps as in \citep{jaksch2010near}, for any $k$ with $M \in \cM_{k}$, we have that:
\begin{align*}
L_2(k) &\leq \sum_{t=t_k}^{t_{k+1} -1} X_t + D\, ,
\end{align*}
so that $\sum_{k=1}^{m(T)} L_2(k) \leq \sum_{t=1}^T X_t +m(T)D$. Therefore, by  Lemma \ref{lem:time-uniform-AzumaHoeffding}, we deduce that with probability at least $1-\delta$,
\begin{align}
\sum_{k=1}^{m(T)} L_2(k) &\leq D \sqrt{2 (T+1) \log(\sqrt{T+1}/\delta)} + m(T)D \sk
&\leq D \sqrt{2 (T+1) \log(\sqrt{T+1}/\delta)}  + DC\log_2(\tfrac{8T}{C})\, ,
\label{eq:trans_prob_second_overall_minmk}
\end{align}
where the last step follows from Lemma \ref{lem:NbEpisodes}.

\paragraph{Final control.}Combing (\ref{eq:main_delta_minmk})--(\ref{eq:trans_prob_second_overall_minmk}) and summing over all episodes give:
\begin{align}
\sum_{k=1}^{m(T)} \Delta_k &\bI\{M \in \cM_{k}\} \leq
\sum_{k=1}^{m(T)} L_1(k) + \sum_{k=1}^{m(T)} L_2(k) + 2\Big(\sqrt{\log\big(C\sqrt{T+1}/\delta\big)} + 1\Big)\sum_{k=1}^{m(T)} \sum_{c \in \cC} \frac{\nu_k(c)}{\sqrt{n_{k}(c)}} \nonumber \\
&\leq 2\Big(D \sqrt{\log\big(C2^S\sqrt{T+1}/\delta\big)} + \sqrt{\log\big(C\sqrt{T+1}/\delta\big)} + 1\Big) \sum_{k=1}^{m(T)}
\sum_{c\in \cC} \frac{\nu_k(c)}{\sqrt{n_{k}(c)}}\nonumber\\
& + D\sqrt{ 2 (T+1) \log\big( \sqrt{T+1}/\delta\big)} + DC\log_2(\tfrac{8T}{C})	 \, ,	
\label{eq:trans_prob_intermed_minmk}
\end{align}
with probability at least $1-\delta$. To upper bound the right-hand side, we recall the following lemma:

\begin{lemma}[{\cite[Lemma~19]{jaksch2010near}}]\label{lem:sequence}
	For any sequence of numbers $z_1, z_2, \dots, z_n$ with $0 \leq z_k \leq Z_{k-1} := \max\{1, \sum_{i=1}^{k-1}z_i\}$,
	$$
	\sum_{k=1}^n 	\frac{z_k}{\sqrt{Z_{k-1}}} 	\leq 	\big(\sqrt{2} + 1\big) 	\sqrt{Z_n}\, .
	$$
\end{lemma}

Note that $n_{k}(c) = \sum_{k'<k} \nu_{k'}(c)$. Hence, applying Lemma \ref{lem:sequence} gives
\begin{equation*}
\sum_{c \in \cC} \sum_{k=1}^{m(T)} \frac{\nu_k(c)}{\sqrt{n_{k}(c)}}  \leq
\sum_{c \in \cC} \big(\sqrt{2} + 1\big) \sqrt{n_{m(T)}(c)} \le	\big( \sqrt{2} + 1 \big) \sqrt{CT}\, ,
\end{equation*}
where the last step follows from Jensen's inequality and $\sum_c n_{m(T)}(c) = T$. Therefore,
\begin{align*}
\sum_{k=1}^{m(T)} \Delta_k \bI\{M \in \cM_{k}\}
&\leq  D\sqrt{ 2 (T+1) \log\big( \sqrt{T+1}/\delta\big)} + DC\log_2(\tfrac{8T}{C}) \nonumber\\
&+ 2\big(\sqrt{2} + 1 \big)\Big(D \sqrt{\log\big(C2^{S}\sqrt{T+1}/\delta\big)} + \sqrt{\log\big(C\sqrt{T+1}/\delta\big)} + 1\Big)  \sqrt{CT}\, ,
\end{align*}
with probability of at least $1 - \delta$. Finally, the regret of \CUCRL($\cC,\sigma$) is controlled on an event of probability higher than $1-2\delta-\delta-\delta$, uniformly over all $T$,
by
\als{
\kR(T) &\leq 2\big(\sqrt{2} + 1\big)\Big(D \sqrt{\log\big(C2^{S}\sqrt{T+1}/\delta\big)} + \sqrt{\log\big(C\sqrt{T+1}/\delta\big)} + 1\Big)  \sqrt{CT}\\
& +D\sqrt{ 2 (T+1) \log\big( \sqrt{T+1}/\delta\big)} + DC\log_2(\tfrac{8T}{C}) + \sqrt{ \tfrac{1}{2}(T+1) \log\big( \sqrt{T+1}/\delta\big)}\\
&\leq 18\sqrt{CT\big(S + \log(C\sqrt{T+1}/\delta)\big)} +  DC\log_2(\tfrac{8T}{C}) \, ,}
thus completing the proof. We finally note that when the mean reward function is known, as in the main text, the above bound holds with a probability higher than $1-3\delta$.
\end{proof}

\subsection{Proof of Lemma \ref{lem:NbEpisodes}}
The proof uses similar steps as in the proof of Proposition 18 in \citep{jaksch2010near}.

Recall that given $c$, $N_T(c)$ and $\nu_k(c):=\nu_{t_k}(c)$ denote as the total number of state-action observations, up to step $T$ and in episode $k$, respectively. 
For any $c$, let $K(c)$ denote the number of episodes where a state-action pair from $c$ is sampled: $K(c)=\sum_{k=1}^{m(T)} \bI\{\nu_k(c)>0\}$. It is worth mentioning that if $n_{k}(c) > 0$ and $\nu_k(c) = n_{k}(c)$, by the design of the algorithm, $n_{k+1}(c) = 2n_{k}(c)$. Hence,
\als{
n_{m(T)}(c) = \sum_{k = 1}^{m(T)} \nu_k(c) \geq 	1 + \sum_{k: \nu_k(c) = n_{k}(c)}
n_{k}(c) 	\geq 	1 + \sum_{i = 1}^{K(c)}2^{i-1}  = 2^{K(c)} \, .
}

If $n_{m(T)}(c) = 0$, then $K(c) = 0$, so that $n_{m(T)}(c) \geq 2^{K(c)} - 1$ for all $c$. Thus,
\begin{equation*}
T = \sum_{c \in \cC} n_{m(T)}(c) \geq \sum_{c \in \cC} \big(2^{K(c)} - 1\big)
\end{equation*}
On the other hand, an episode has happened when either $n_{k}(c) = 0$ or $n_k(c) = \nu_k(c)$. Therefore, $m(T) \leq 1 + C + \sum_{c \in \cC} K(c)$ and consequently, $\sum_{c \in \cC} K(c) \geq m(T) - 1 - C$. Hence, by Jensen's inequality, we obtain
\begin{equation*}
\sum_{c \in \cC}2^{K(c)} \geq C2^{\sum_{c \in \cC}\frac{K(c)}{C}} \geq C2^{\frac{m(T)-1}{C}-1} \, .
\end{equation*}
Putting together, we obtain $T \geq C\big(2^{\frac{m(T)-1}{C} - 1} - 1\big)$. Therefore,
\begin{equation*}
m(T) \leq 1 + 2C + C\log_2(\tfrac{T}{C}) \leq 3C + C\log_2(\tfrac{T}{C}) \leq C\log_2(\tfrac{8T}{C})\, ,
\end{equation*}
thus concluding the proof.
\ep

\section{Environments Used in Numerical Experiments}\label{app:simul}
In this section, we provide further details for the environments used in numerical experiments in Section \ref{sec:xps}.

\subsection{RiverSwim and Ergodic RiverSwim}
In the first set of experiments, we examined the performance of various algorithms in \emph{RiverSwim} environments. Figures \ref{fig:commu_river_swim} and \ref{fig:ergodic_river_swim} respectively display the $L$-state \emph{RiverSwim} and ergodic \emph{RiverSwim} environments.

\begin{figure}[!th]
\centering
\scriptsize
\def\svgwidth{0.7\columnwidth}
\input{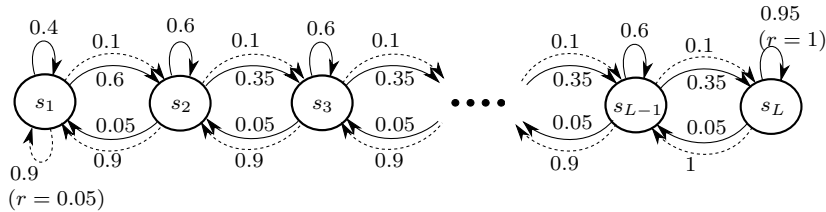}
\caption{The $L$-state \emph{Ergodic RiverSwim} MDP}
\label{fig:ergodic_river_swim}
\end{figure}

\begin{figure}[!th]
    \begin{center}
    \scriptsize
    \def\svgwidth{0.7\columnwidth}
	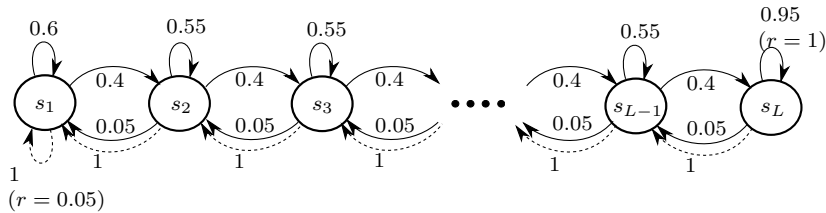
	\caption{The $L$-state \emph{RiverSwim} MDP}
    \label{fig:commu_river_swim}
\end{center}
\end{figure}

\subsection{Grid-World}

We conducted our last set of experiments in a $7 \times 7$ grid-world environment shown in Figure \ref{fig:4-room}, which we refer to as the 4-room grid-world. This MDP comprises 20 states ($S=20$). In this environment, the initial state is the upper-left corner (shown in red). When the learner reaches the lower-right corner (shown in yellow), a reward of 1 is given, and the learner is sent back to the initial state. The learner can perform four actions ($A=4$): Going up, left, down, or right. After playing a given action, the learner stays in the same state with probability 0.1, moves to the desired direction with probability 0.7 (for example, to the left, if the learner chooses to `go left'), and moves to other possible directions with probability 0.2. Walls act as \emph{reflectors}: When the next state is a wall, the transition probability of it is added to that of the current state.

\begin{figure}[!th]
	 \begin{center}
		\includegraphics[width=.5\linewidth]{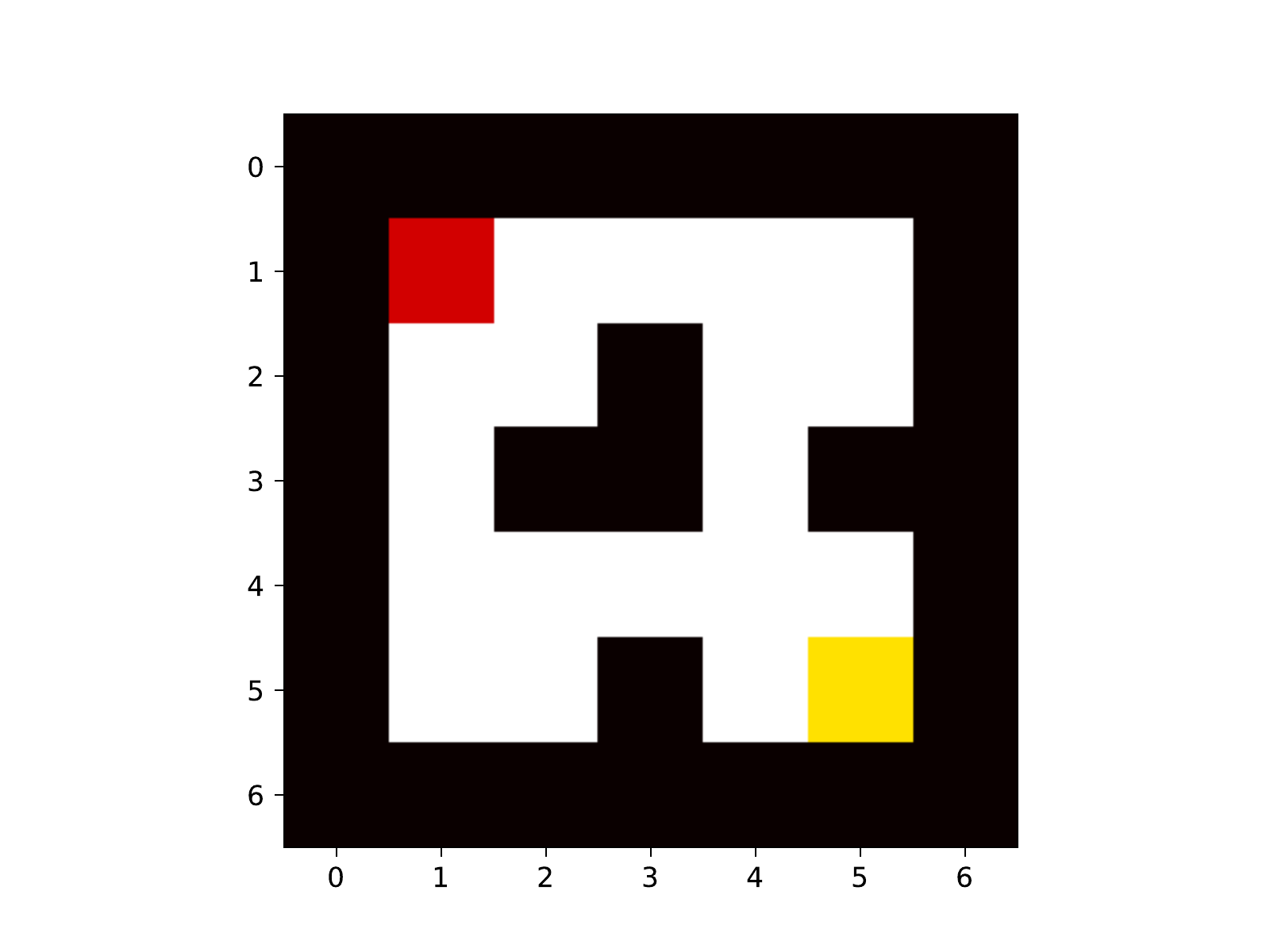}
    \caption{The 4-room grid-world MDP}
		\label{fig:4-room}
        \end{center}
\end{figure}

\section{Other Examples of MDPs}\label{app:examples}

In this section, we examine the notion of equivalence presented in Definition~\ref{def:equi_class} on some grid-world environments.

For this purpose, we consider grid-world MDPs. The action-space is $\{ u,d,l,r \}$.
Playing action $a=u$ moves the current state `up' with probability $0.8$, does not change the current state with probability $0.1$, and moves left or right with the same probability $0.05$. Walls act as \emph{reflectors}: When the next state is a wall, the transition probability of it is added to that of the current state. Other actions are defined in a similar way. Finally, the goal-state is put in the bottom-right corner of the MDP, where the learner is given a reward of 1.

Below, we show four examples of grid-world environments defined according to the above scheme, with different numbers of state-action pairs. The number of state-action pairs in the introduced 4-room and 2-room MDPs changes as the grid size grows, while keeping the number of classes almost fixed:

\begin{center}
\begin{tabular}{|c|c|c|c|c|}
\hline
	grid-world &
	Figure~\ref{fig:comp1} &
	Figure~\ref{fig:comp2} &
	Figure~\ref{fig:comp3} &	
	Figure~\ref{fig:comp4}\\ \hline
	$SA$ & $84$ & $800$ & 736 & $\sim 10^4$\\ \hline
	$C$ &6 & 6 & 7 &	7\\ \hline
\end{tabular}
\end{center}

\begin{center}
	\begin{tabular}{|c|c|c|c|c|c|} \hline
		Environment & States& $5\times 5$ & $7\times 7$ & $9\times 9$ & $100\times 100$ \\ \hline
		4-Room & $SA$ & $100$ & $196$ & $324$ &$4\times 10^4$\\ \hline
		4-Room & $C$ & 3& 3 & 3 & 3\\ \hline
		2-Room & $SA$ & $100$ & $196$ & $324$ &$4\times 10^4$\\ \hline
		2-Room & $C$ & 4& 4 & 4 & 4\\		\hline
	\end{tabular}
\end{center}
We stress that other notions of similarity from the RL literature do not scale well. For instance, in \citep{ortner2013adaptive}, a partition $\cS_1,\dots\cS_n$ of the state-space $\cS$ is considered to define an aggregated MDP, which satisfies: For all $i\in \{1,\ldots,n\}$,
\beqan
\forall s,s' \in\cS_i, \forall a\in\cA, \qquad   &\mu(s,a)= \mu(s',a)\,,\\
\forall j, \qquad  &\sum_{s'' \in\cS_j}p(s''|s,a) = \sum_{s'' \in\cS_j}p(s''|s',a)\,.
\eeqan
This readily prevents any two states $s,s'$ such that $p(\cdot|s,a)$ and $p(\cdot|s',a)$ have disjoint supports from being in the same set $\cS_i$.
Thus, since in a grid-world MDP, where transitions are local, the number of pairs with disjoint support is (almost linearly) increasing with $S$, this implies a potentially large number of classes for grid-worlds with many states.
A similar criticism can be formulated for \citep{anand2015asap}, even though it considers sets  of state-action pairs instead of states only, thus slightly reducing the total number of classes.

\begin{figure}[!hbtp]
	\includegraphics[width=0.5\textwidth]{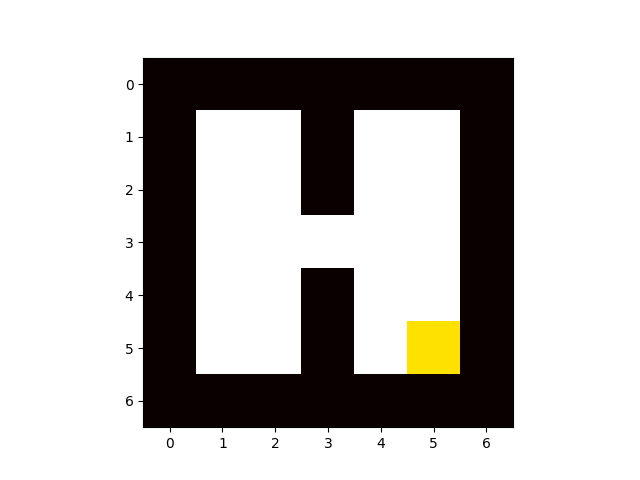}
	\hfill	
	\includegraphics[width=0.5\textwidth]{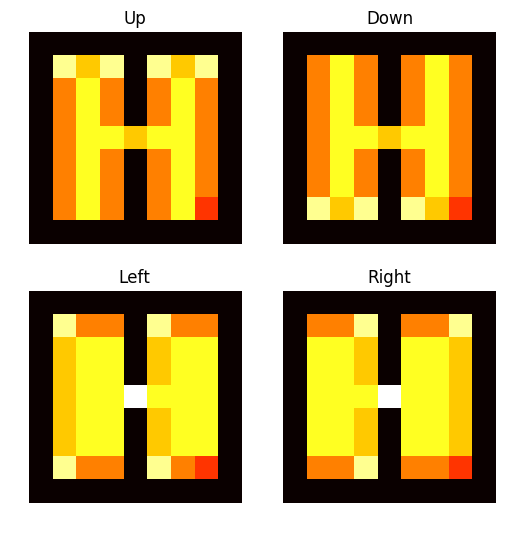}
	
	\caption{
		Left: Two-room grid-world (left) with walls in black, and goal state in yellow.
		Right: equivalence classes for state-action pairs (one color per class).
	}
\label{fig:comp1}
\end{figure}

\begin{figure}[!hbtp]
	\includegraphics[width=0.5\textwidth]{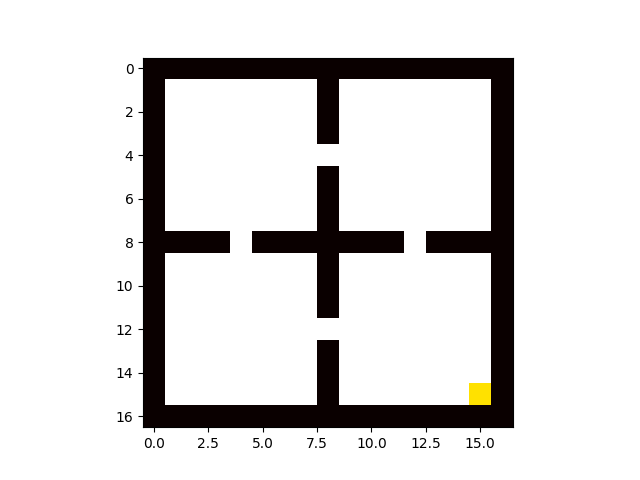}
	\hfill	
	\includegraphics[width=0.5\textwidth]{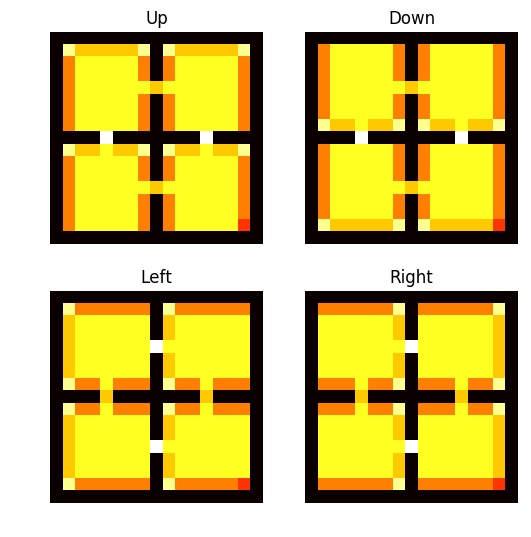}
	
	\caption{Left: Four-room grid-world (left) with walls in black, and goal state in yellow.
		Right: equivalence classes for state-action pairs (one color per class).
	}\label{fig:comp2}
\end{figure}

\begin{figure}[!hbtp]
	\includegraphics[width=0.5\textwidth]{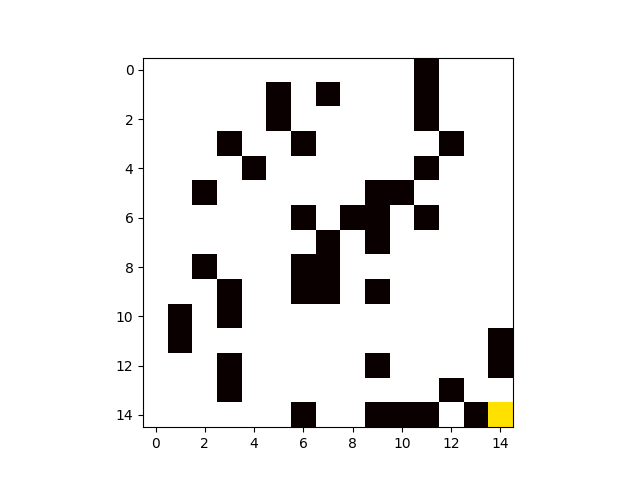}
	\hfill	
	\includegraphics[width=0.5\textwidth]{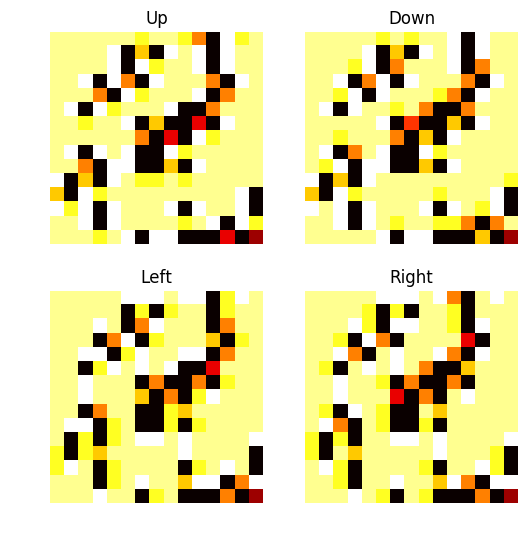}

	\caption{Left: A more complex grid-world (left) with walls in black, and goal state in yellow.
		Right: equivalence classes for state-action pairs (one color per class).}
\label{fig:comp3}
\end{figure}

\begin{figure}[!hbtp]
	\includegraphics[width=0.5\textwidth]{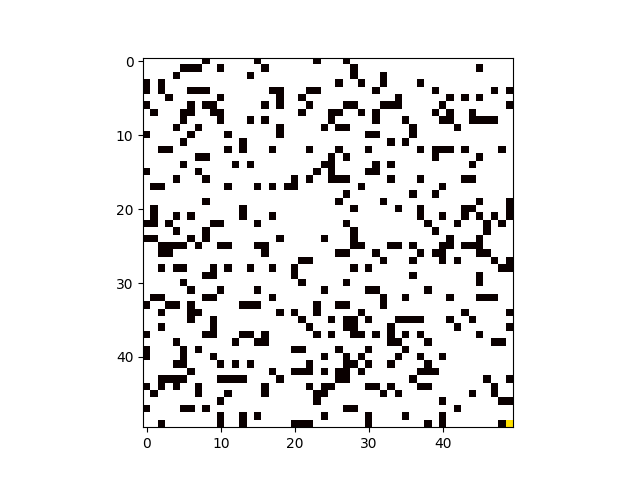}
	\hfill	
	\includegraphics[width=0.5\textwidth]{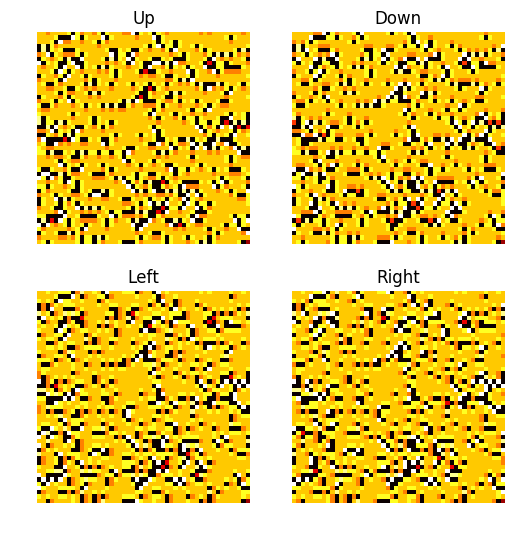}
		\caption{Left: A more complex grid-world (left) with walls in black, and goal state in yellow.
		Right: equivalence classes for state-action pairs (one color per class).
	}\label{fig:comp4}
\end{figure}

\newpage
\bibliography{ACML2019_short_bib}

\end{document}